\documentclass[sigconf]{acmart}

\def\thisismainpaper{0}   

\acmYear{2024}\copyrightyear{2024}
\setcopyright{acmlicensed}
\acmConference[MobiHoc '24]{International Symposium on Theory, Algorithmic Foundations, and Protocol Design for Mobile Networks and Mobile Computing}{October 14--17, 2024}{Athens, Greece}
\acmBooktitle{International Symposium on Theory, Algorithmic Foundations, and Protocol Design for Mobile Networks and Mobile Computing (MobiHoc '24), October 14--17, 2024, Athens, Greece}
\acmDOI{10.1145/3641512.3686365}
\acmISBN{979-8-4007-0521-2/24/10}

\usepackage{graphicx}
\usepackage{amsmath}
\usepackage{multirow}
\usepackage{epsfig}
\usepackage{mathrsfs}
\usepackage{comment}
\usepackage{hyperref}
\usepackage{dsfont}
\usepackage{array}
\usepackage{amsthm}
\usepackage{blkarray}
\usepackage{enumerate}
\usepackage{url}
\usepackage{chemarrow}
\usepackage{float}
\usepackage{color}
\usepackage[lined,ruled,linesnumbered,noend]{algorithm2e}
\SetKwRepeat{Do}{do}{while}%
\usepackage{epsf,psfrag}
\usepackage{epsfig}
\usepackage{bbm}
\usepackage{bm}

\def\rev#1{\textcolor{black}{#1}}

\theoremstyle{definition}
\newtheorem{theorem}{Theorem}[section]
\newtheorem{lemma}[theorem]{Lemma}
\newtheorem{corollary}[theorem]{Corollary}

\newtheorem{definition}{Definition}

\newcommand{\E}{{\rm I\kern-.3em E}}

\def\xbf{\bm{x}}
\def\wbf{\bm{w}}
\def\xibf{\bm{\xi}}
\def\ubf{\bm{u}}

\def\mbbE{\mathbb{E}}
\DeclareMathOperator{\argmin}{\arg\min}

\def\figheight{1.15in}

\begin{document}

\title{Active Learning for WBAN-based Health Monitoring}
\author{Cho-Chun Chiu, Tuan Nguyen, Ting He}
\if\thisismainpaper0
\authornote{This is the extended version with proofs and supporting materials.}
\fi
\affiliation{%
  \institution{Pennsylvania State University}
  \streetaddress{}
  \city{University Park}
  \state{PA}
  \country{USA}
  \postcode{}
}
\email{{cuc496,tmn5319,tzh58}@psu.edu}

\author{Shiqiang Wang}
\affiliation{%
  \institution{IBM T. J. Watson Research Center}
  \streetaddress{}
  \city{ Yorktown Heights}
\state{NY}  
\country{USA}}
\email{wangshiq@us.ibm.com}

\author{Beom-Su Kim, Ki-Il Kim}
\authornote{Beom-Su Kim and Ki-Il Kim were supported by Basic Science Research Program through the National Research Foundation of Korea (NRF) funded by the Ministry of Education (RS-2023-00237300) and Ministry of Science and ICT (RS-2022-00165225).}
\affiliation{%
  \institution{Chungnum National University}
  \city{Daejeon}
  \country{Korea}
}
\email{{bumsou10,kikim}@cnu.ac.kr}

\renewcommand{\shortauthors}{Chiu et al.}

\begin{abstract}
We consider a novel active learning problem motivated by the need of learning machine learning models for health monitoring in wireless body area network (WBAN). Due to the limited resources at body sensors, collecting each unlabeled sample in WBAN incurs a nontrivial cost. Moreover, training health monitoring models typically requires labels indicating the patient's health state that need to be generated by healthcare professionals, which cannot be obtained at the same pace as data collection. These challenges make our problem fundamentally different from classical active learning, where unlabeled samples are free and labels can be queried in real time. To handle these challenges, we propose a two-phased active learning method, consisting of an online phase where a coreset construction algorithm is proposed to select a subset of unlabeled samples based on their noisy predictions, and an offline phase where the selected samples are labeled to train the target model. The samples selected by our algorithm are proved to yield a guaranteed error in approximating the full dataset in evaluating the loss function. Our evaluation based on real health monitoring data and our own experimentation demonstrates that our solution can drastically save the data curation cost without sacrificing the quality of the target model. 
\looseness=-1
\end{abstract}

\ccsdesc[500]{Computing methodologies~Active learning settings}
\ccsdesc[300]{Networks~Wireless personal area networks}

\keywords{Active learning, WBAN, coreset. }


\settopmatter{printfolios=true} 
\maketitle

\section{Introduction}\label{sec:Introduction}

By actively selecting which samples to use in training, \emph{active learning} can significantly improve the sample efficiency of supervised learning, which brings tremendous benefits in scenarios where acquiring labeled training data is expensive. In this work, we study active learning under  novel challenges motivated by its application in health monitoring based on \emph{wireless body area network (WBAN)}. \looseness=-1 

\begin{figure}[!t]
   \centerline{\includegraphics[width=0.8\linewidth]{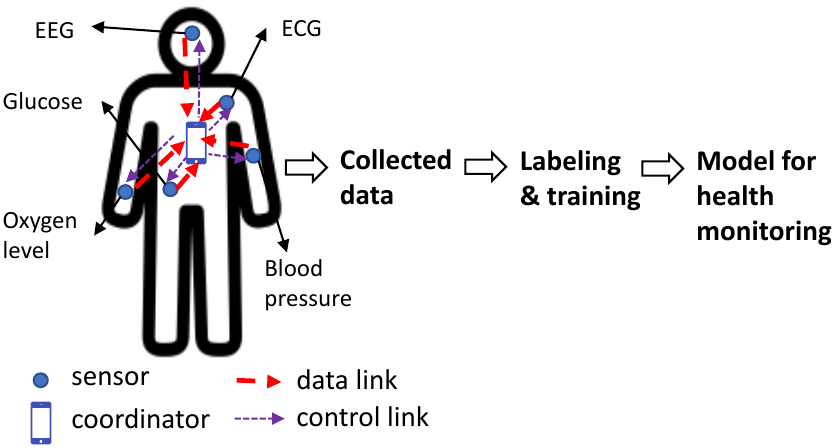}}
   \vspace{-1em}
    \caption{Workflow for active learning in WBAN. }
    \label{fig:WBAN}
    \vspace{-.5em}
\end{figure}

As illustrated in Fig.~\ref{fig:WBAN}, a WBAN is a personal sensor network comprised of a \emph{coordinator} (e.g., smart phone) and a set of \emph{on-body/in-body sensors} measuring various physiological signals, connected by wireless links following a certain protocol (e.g., IEEE 802.15.6). 
WBAN-based health monitoring has attracted notable interests as a promising technology to enable personalized healthcare~\cite{Sharma11BMEI}, but faces critical limitations due to the limited resources at the sensors, particularly the limited power supply. While there have been efforts on improving the energy efficiency of WBAN~\cite{Kim22TMC}, 
their focus has been on low-level performance measures such as communication delays, but the actual performance of WBAN-based applications has been largely ignored. In this work, we address this gap from the perspective of active learning. 
Specifically, we aim at improving the efficiency of WBAN in providing the data for training/refining a target machine learning model of interest. 
The target model is typically a classifier used to monitor the patient's health state from physiological metrics collected by body sensors such as heart rate, blood oxygen level, etc. 
This problem has several fundamental differences from classical active learning problems: (i) instead of real-time labeling of selected samples, the labeling process for health monitoring usually requires human intervention (e.g., annotations by healthcare professionals) that cannot be done at the pace of data collection; (ii) instead of having free access to all the unlabeled samples, WBAN incurs a nontrivial cost (e.g., sensor energy consumption) in collecting each unlabeled sample. Moreover, the data sources in WBAN (i.e., body sensors) are typically proprietary devices with little support for programming, 
which restricts the solution space to actions that can be performed by the coordinator.  \looseness=-1

In this work, we address these challenges by developing a two-phased active learning method, the core of which is an algorithm deployed on the coordinator that can selectively collect unlabeled samples from remote data sources based on their noisy predictions without querying the labels, while having guaranteed training performance. Although motivated by the application scenario of WBAN, our solution is applicable in any active learning scenario where its assumed input is available.

\subsection{Related Work}\label{subsec:Related Work}

\textbf{WBAN.} WBAN is designed to monitor physiological signals of the target user through on-body/in-body sensors that report to a coordinator with data processing capabilities \cite{Sharma11BMEI}. Due to the limited capabilities of body sensors, the function of WBAN focuses on enabling sensors to report their measurements to the coordinator, after which the rest of the processing/reporting will be handled by the coordinator. The industry standard for communications within WBAN is IEEE 802.15.6~\cite{WBANstandard}, which details a number of specifications and performance requirements. Achieving all these requirements within the hardware limitations raises many challenging research questions, which have attracted extensive studies as surveyed in \cite{Kim22AHN}. Many of these studies focused on improving the tradeoff between energy efficiency and low-level performance measures, such as
\if\thisismainpaper1
\cite{Kim22TMC}
\else
\cite{Sun21SJ,Kim22TMC} 
\fi
and references therein, but the actual performance of applications based on the data collected by WBAN has been largely ignored. \emph{In this work, we address this gap for the application of active learning for health monitoring.}\looseness=-1

\textbf{Active learning.} Active learning aims at training the target model to sufficient accuracy with the minimum labeling cost by only querying the labels for ``informative samples''~\cite{Ren21Survey}. There has been a body of works on active learning, briefly summarized below.  
One approach to active learning is uncertainty sampling \cite{David94}, which  measures the informativeness of a sample by its uncertainty. 
This approach typically draws unlabeled samples at random and only labels those that fall within an uncertainty region of the currently trained model, where the uncertainty is usually measured by entropy \cite{Shannon01} or confidence \cite{Culotta05}. 
Another approach to active learning exploits disagreement between models. The Query by Committee (QBC) strategy \cite{Seung92} involves training multiple models on the available labeled data and selecting informative samples as those producing high disagreement among the models. 
In addition to uncertainty and disagreement, informativeness can also be measured by the expected gradient length (EGL) \cite{Settles08}. EGL considers the sample resulting in a training gradient of the largest magnitude as the most informative, but since the label is unknown, it selects the sample with the largest expected gradient magnitude. 
All the above approaches tend to query outliers. However, truly informative samples should be not only sufficiently distinct but also representative. Accordingly, the information density framework \cite{Settles08EMNLP} include a density term with controllable importance. 
Recent studies on active learning focused on the support of deep learning. For instance, \cite{ChristopherACL22} investigated the effectiveness of uncertainty sampling in support of transformer-based models. 
\if\thisismainpaper0
On the theoretical front, \cite{Wang21NeurIPS} proposed an algorithm for active learning of over-parameterized deep neural networks, which provided performance guarantee for the trained model based on the neural tangent kernel (NTK) approximation.
\fi
These approaches, however, are not applicable in our setting, as they required access to all unlabeled samples  and real-time query of the labels. 
\emph{In this regard, we address a novel active learning problem where one has to select unlabeled samples before knowing their (exact) values or being able to start the training of the target model.}

\textbf{Coreset.} Coreset is a small weighted dataset in the same space as the original dataset that approximates the original dataset in terms of a given loss function~\cite{Munteanu18KI}. Initial applications of coreset focused on accelerating the computation for selected problems in unsupervised learning (e.g., clustering~\cite{Badoiu02STOC}) or simple supervised learning (e.g., support vector machine \cite{
Har-Peled07IJCAI}). Later, due to its capability of reducing the number of samples, coreset started to be applied to communication reduction in distributed learning~\cite{Lu20JSAC}, and labeling cost reduction in active learning~\cite{Sener18ICLR}. 
Our work is closest to \cite{Sener18ICLR} in that we also convert our active learning problem into a coreset construction problem on an unlabeled dataset. However, the solution in \cite{Sener18ICLR} is not suitable for active learning in WBAN because: (i) it requires perfect knowledge of all the unlabeled samples; (ii) it requires the samples to be i.i.d. in time, which is often violated in WBAN where samples exhibit strong temporal correlation and heterogeneity; (iii) it requires the target model to achieve zero loss on the training samples from the coreset, which will force the target model to overfit the training data and cause poor generalization error. We will address all these limitations by 
\rev{developing a new theoretical foundation that can handle non-i.i.d. samples and non-zero training loss on the coreset, as well as a predictive coreset construction algorithm designed to adjust for prediction errors.}
\looseness=-1

\subsection{Summary of Contributions}

Motivated by personalized healthcare via WBAN, we study an active learning problem of selectively collecting data from remote data sources to train a target model based on noisy predictions of unlabeled samples, with the following contributions:\looseness=-1

1) By deriving a generalized error bound on coreset-based model training, we convert the active learning problem to a coreset construction problem, which allows us to select samples without querying their labels\footnote{The selected samples still need to be labeled before starting training the target model, but no label is required during the sampling process.}. Our generalized bound reveals a novel tradeoff between the coverage radius and the number of samples covered by each point in the coreset.  

2) Motivated by the error bound, we develop a sampling algorithm that performs predictive coreset construction based on given noisy predictions of unlabeled samples, and analyze its performance in terms of the dominant terms in the error bound as a function of design parameters and prediction error. Only samples selected into the coreset are collected, after which we have a classical pool-based active learning problem which can be solved by any existing solution. \looseness=-1

3) We evaluate our solution based on both a public dataset and our own prototype implementation. The results show that: (i) it is possible to meaningfully predict body sensor measurements, (ii) based on such noisy predictions, our algorithm can produce a target model of almost the same quality as the model trained on the full data, while reducing the data collection and labeling cost by $>90\%$, and (iii) the results are consistent for  different users. \looseness=0

\textbf{Roadmap.} Section~\ref{sec:Background and Formulation} formulates our problem, Section~\ref{sec:Theoretical Foundation} provides a theoretical foundation, Section~\ref{sec:Algorithm Design} presents our algorithm design and analysis, Section~\ref{sec:Performance Evaluation} evaluates our solution, and Section~\ref{sec:Conclusion} concludes the paper. 
\if\thisismainpaper1
\textbf{All the proofs are given in \cite{Chiu24:report}.}
\else
\textbf{Proofs and supporting materials are given in Appendix~\ref{appendix:A}.} 
\fi



\section{Background and Formulation}\label{sec:Background and Formulation}

\subsection{Target Model}

We assume that the health monitoring application is interested in learning a classification model, with a $d$-dimensional input space $\mathcal{X}\subseteq \mathbb{R}^d$ and a finite output space $\mathcal{Y} := \{1,\ldots,C\}$. 
Each sample $\xbf\in \mathcal{X}$ represents the concatenation of all sensor measurements in a sampling epoch, and each label $y\in \mathcal{Y}$ represents a health state of the user. 
Let $\eta(\xbf):=(\eta_c(\xbf))_{c\in \mathcal{Y}}$ denote the \emph{unknown} regression function specifying the conditional distribution of the label $y$ of a given sample $\xbf$, where $\eta_c(\xbf):=\Pr\{y=c|\xbf\}$. The goal is to find the best approximation of $\eta(\xbf)$ within a given set of hypotheses $\{h(\xbf; \wbf)\}_{\wbf\in \mathcal{W}}$ such that the expected error over all possible values of $\xbf$ according to a given loss function $l(\xbf,y;\wbf)$ is minimized, where the parameter $\wbf$ of the selected hypothesis is the learned model parameter (e.g., the vector of link weights in a neural network). \looseness=-1

\emph{Remark:} In practice, there is often an initial model trained on public data, and the problem is about fine-tuning this model on the personal data of the current user collected through WBAN.

\subsection{Background on WBAN}

According to the IEEE 802.15.6 standard, a WBAN consists of a coordinator and a set of on-body/in-body sensors, interconnected through a star topology as illustrated in Fig.~\ref{fig:WBAN}. We assume that the coordinator adopts the beacon mode with superframe as the access mode~\cite{Kim22TMC}. In this mode, the coordinator schedules uplink transmissions at the granularity of timeslots in each superframe, and notifies the sensors of the scheduling decision via a beacon frame (which includes the scheduling information). After receiving the beacon frame, each scheduled sensor transmits the latest sampled data stored in its buffer to the coordinator during the scheduled timeslots. Through this procedure, the coordinator can control the collection of \emph{unlabeled} data. Note that to comply with the IEEE 802.15.6 standard, we only assume that the coordinator can control the transmissions of samples, but not the generation of samples. While this means that our solution can only save the transmission cost but not the sampling cost at sensors, it ensures the broad applicability of our solution as it does not directly program the sensors. If needed, our solution can be easily adapted for deployment on the data sources themselves (see the remark in Section~\ref{subsec:Sampling by Predictive Coreset Construction}). 

\if\thisismainpaper0
\emph{Remark:} Although the coordinator can change the allocation of timeslots among the sensors in each superframe, doing so will substantially complicate the solution space. For tractability, we restrict the coordinator's action in each superframe to either scheduling transmissions from all the sensors of interest (according to a fixed timeslot allocation) or not scheduling any transmission, leaving fine-grained optimization of the timeslot allocation to future work. \looseness=-1 
\fi

\subsection{Background on Classical Active Learning}

Classical active learning aims at reducing the labeling cost by selecting a subset of the given unlabeled samples to label. Specifically, given a set of \emph{unlabeled} samples $\{\xbf_i\}_{i=1}^n$, an initial set of \emph{labeled} samples $S^{(0)}$, a budget $b_l$ of queries to a labeling oracle, and a training algorithm that returns a trained model $\pi_S$ based on a given training set $S$ and the corresponding labels, an active learning algorithm seeks to select up to $b_l$ unlabeled samples to query so that the resulting model has the minimum expected loss, i.e., 
\begin{align}
    \min_{S^{(1)}\subseteq \{\xbf_i\}_{i=1}^n: |S^{(1)}|\leq b_l} \mbbE_{\xbf,y}[l(\xbf,y; \pi_{S^{(0)}\cup S^{(1)}})].
\end{align}
While the above description is for the offline setting where all the unlabeled samples are given, it can be applied iteratively in the online setting, where batches of unlabeled samples arrive sequentially and the labeling decisions have to be made sequentially too (in this case $S^{(0)}$ denotes the set of samples that have been labeled so far).

\subsection{Problem: Active Learning in WBAN}\label{subsec:Active Learning Problem in WBAN}

\rev{The personalized nature of health monitoring makes it necessary to train/fine-tune the model for each user based on its own data. However, collecting such data via WBAN will incur nontrivial cost, particularly in terms of sensor energy consumption. This observation motivates the application of active learning in WBAN, even if the labeling cost is negligible.}
The approach of classical active learning fails to meet the requirements of WBAN: (i) instead of a labeling oracle that can be queried in real time, the labeling of WBAN samples (e.g., judgement of the user's health status) usually requires human intervention and cannot be performed at the pace of sampling; (ii) instead of accessing all the unlabeled samples for free, WBAN incurs a nontrivial cost (e.g., energy consumption at sensors) in collecting each unlabeled sample at the coordinator, where the active learning algorithm runs. Therefore, an efficient active learning algorithm in WBAN should selectively collect a subset of samples that, once labeled, are most useful for training the target model, without being able to query the labels and start the training before the data collection is done. Fig.~\ref{fig:WBAN} illustrates the workflow for active learning in WBAN considered in this work. 


To address the above challenges, we propose to perform active learning through \emph{coreset construction}~\cite{Lu20JSAC}, which selects representative samples based on their positions in the input space $\mathcal{X}$, thus circumventing the need of real-time labeling. Doing so at the coordinator, however, faces a dilemma that the algorithm needs to know what the samples will look like {before} taking those samples. Fortunately, health monitoring data usually exhibit strong temporal correlations such as daily patterns, and the coordinator itself (usually a smart phone) is also equipped with local sensors whose measurements are spatially correlated with those of body sensors. We propose to exploit such spatial-temporal correlations to predict future samples through time series forecasting~\cite{Lim21Survey}. 
Our focus in this work is \emph{not} on developing new time series forecasting models. Instead, we focus on optimizing the selection of samples based on their predicted values given by a forecasting model while taking into account the prediction errors, and our solution can be used in combination with any reasonable forecasting model. 

Specifically, given an initial set of \emph{unlabeled} samples $S^{(0)}$ and a forecasting model that can predict $n$ unlabeled samples at a time, our active learning framework contains two phases:\begin{enumerate}
    \item \emph{online sampling phase} that works in multiple rounds, where in each round, we predict the next batch of samples $\hat{P}:= \{\hat{\xbf}_i\}_{i=1}^n$, select a subset $\hat{S}^{(1)}\subseteq \hat{P}$ of the predicted samples, collect the corresponding actual samples $S^{(1)}:= \{\xbf_i:\: \hat{\xbf}_i\in \hat{S}^{(1)}\}$, and merge them into the existing sample set $S^{(0)}$; 
    \item \emph{offline labeling phase}, where after online sampling, we label the collected dataset to generate a labeled training set for training the target model.  
\end{enumerate}
In the sequel, we will only focus on the online sampling phase, assuming that all the collected samples will be labeled.   
As the problem will reduce to a classical pool-based active learning problem after online sampling, in practice one can apply any pool-based active learning algorithm to the dataset collected by our algorithm to further reduce the labeling cost. 

Our goal is to develop an online sampling algorithm with good tradeoff between the number of collected samples and the performance in target model training.  
To provide a theoretical performance guarantee, we make the following \textbf{assumptions}: \looseness=-1
\begin{enumerate}
    \item Given the predicted samples $\hat{P}:= \{\hat{\xbf}_i\}_{i=1}^n$, the prediction errors $\{\xbf_i-\hat{\xbf}_i\}_{i=1}^n$ are conditionally i.i.d., each following the Gaussian distribution $\mathcal{N}(\bm{0}, \sigma_n^2 \bm{I})$, where $\sigma_n^2$ denotes the prediction mean squared error (MSE). 
    \label{assumption:prediction error}
    \item The loss function $l(\xbf,y; \wbf)$ is upper-bounded by $L$ for all $\xbf\in \mathcal{X}$, $y\in \mathcal{Y}$, and $\wbf\in \mathcal{W}$. \label{assumption:bounded loss}
    \item The loss function $l(\xbf,y; \wbf)$ is $\lambda_l$-Lipschitz continuous in $\xbf$ for all $y\in \mathcal{Y}$ and $\wbf\in \mathcal{W}$. \label{assumption:Lipschitz loss}
    \item The regression function $\eta_c(\xbf)$ is $\lambda_\eta$-Lipschitz continuous in $\xbf$ for all $c\in \mathcal{Y}$. \label{assumption:Lipschitz regression}
\end{enumerate}
Assumptions \eqref{assumption:bounded loss}--\eqref{assumption:Lipschitz regression} are standard assumptions for applying coreset to active learning~\cite{Sener18ICLR}. Assumption~\eqref{assumption:prediction error} is a simplifying assumption to allow closed-form analysis (see Section~\ref{subsubsec:Performance Analysis}). However, our algorithm does not hinge on these assumptions and will be tested on more realistic cases (see Section~\ref{sec:Performance Evaluation}).

\section{Theoretical Foundation}\label{sec:Theoretical Foundation}

As explained in Section~\ref{subsec:Related Work}, we cannot apply classical active learning algorithms due to the lack of real-time access to labels. Instead, we take a coreset-based approach that can directly select unlabeled samples based on the following analysis. 

\subsection{Generalized Approximation Error Bound}

Given a set of unlabeled samples $P:=\{\xbf_i\}_{i=1}^n$, let $S$ be a weighted set such that each point $\xbf_i\in P$ is represented by a point $\xbf_{j_i}\in S$. The weight $u_j$ for each $\xbf_j\in S$ denotes the number of points in $P$ that are represented by $\xbf_j$. For any target model $\wbf$, we can bound the difference between the losses on $P$ and $S$ (after both of them are labeled) as follows.  \looseness=-1


\begin{theorem}\label{thm:approx error bound}
Given $n$ unlabeled samples $P:=\{\xbf_i\}_{i=1}^n$ and a corresponding weighted set $S$ with\footnote{Throughout this paper, $\|\xbf\|$ denotes the $\ell$-2 norm of vector $\xbf$.} $\|\xbf_i-\xbf_{j_i}\|\leq \delta,\: \forall \xbf_i\in P$, if the labels $y_1,\ldots,y_n$ of $P$ are conditionally independent given the unlabeled samples, the loss function satisfies assumptions \eqref{assumption:bounded loss}--\eqref{assumption:Lipschitz loss}, and the regression function satisfies assumption \eqref{assumption:Lipschitz regression}, then with a probability of at least $(1-\gamma)^2$, 
\begin{align}
    {1\over n}\sum_{i=1}^n l(\xbf_i,y_i; \wbf) \leq &\delta(\lambda_l+\lambda_\eta LC)+\sqrt{{L^2\log(2/\gamma)\over 2n}} \nonumber\\
    &\hspace{-7em} +\sqrt{{L^2\log(2/\gamma)\over 2n^2}\sum_{j: \xbf_j\in S} u_j^2} + \sum_{j: \xbf_j\in S} {u_j\over n}l(\xbf_j,y_j;\wbf) \label{eq:approx error bound}
\end{align}
for any given $\wbf\in \mathcal{W}$. 
\end{theorem}

\emph{Remark~1:}
Theorem~\ref{thm:approx error bound} generalizes two limiting assumptions in \cite[Theorem~1]{Sener18ICLR}: (i) \cite{Sener18ICLR} requires the samples $\{(\xbf_i,y_i)\}_{i=1}^n$ to be i.i.d., which is inapplicable in WBAN due to the temporal correlation and heterogeneity in sensor measurements; (ii) \cite{Sener18ICLR}  requires the target model to achieve zero loss on the  samples in $S$ (i.e., $l(\xbf,y;\pi_S)=0,\: \forall \xbf \in S$), which forces the target model to overfit the training data. Instead, we only assume the labels to be conditionally independent once the unlabeled samples are given, and allow the trained model to incur an arbitrary loss over $S$. 

\emph{Remark~2:} 
Theorem~\ref{thm:approx error bound} corrects two mistakes in \cite[Theorem~1]{Sener18ICLR} under its assumptions. Specifically, when the trained model $\pi_S$ achieves zero loss on $S$, 
the bound in Theorem~\ref{thm:approx error bound} reduces to\looseness=-1
\begin{align}
    {1\over n}\sum_{i=1}^n l(\xbf_i,y_i; \pi_S) \leq &\delta(\lambda_l+\lambda_\eta LC)+\sqrt{{L^2\log(2/\gamma)\over 2n}} \nonumber\\
    &\hspace{0em} +\sqrt{{L^2\log(2/\gamma)\over 2n^2}\sum_{j: \xbf_j\in S} u_j^2},
\end{align}
which differs from the bound in \cite[Theorem~1]{Sener18ICLR} in the second and the third terms. The former is due to a mistake in the proof of \cite[Theorem~1]{Sener18ICLR} when applying Hoeffding's bound, and the latter is due to a mistake in the same proof that treated $\mbbE_{y_j}[l(\xbf_j,y_j;\pi_S)] := \sum_{c=1}^C \eta_c(\xbf_j)l(\xbf_j,c;\pi_S)$ as zero for all $\xbf_j\in S$. \footnote{This is invalid, because $\pi_S$ is trained on $(\xbf_j,c_j)$ for a specific realization $c_j$ that the random variable $y_j$ has taken when labeling $\xbf_j$, and thus can only achieve $l(\xbf_j,c_j;\pi_S)=0$ but not $l(\xbf_j,c;\pi_S)=0$ for all $c\in \mathcal{Y}\setminus \{c_j\}$.} \looseness=-1

\subsection{Implication on Coreset Construction} 

The weighted set $S$ is referred to as a \emph{coreset} of the complete sample set $P$ in the sense that after labeling, it achieves a bounded error in approximating the loss function evaluated on $P$. 
Among the terms in the error bound \eqref{eq:approx error bound}, the term $\sqrt{{L^2\log(2/\gamma)/ (2n)}}$ will be negligibly small for a large sample size $n$, and the term 
$\sum_{j: \xbf_j\in S} {u_j\over n} l(\xbf_j,y_j;\wbf)$, 
which denotes the training loss over the coreset $S$, will also be small after training $\wbf$ based on $S$ 
for an expressive target model such as a deep neural network. 
Thus, the coreset $S$ mainly affects the error bound through the following function
\begin{align}
    &\hspace{0em}\delta(\lambda_l+\lambda_\eta LC) + \sqrt{{L^2\log(2/\gamma)\over 2n^2}\sum_{j: \xbf_j\in S} u_j^2} 
    ~\propto \delta + \nu \|\tilde{\ubf}\|, \label{eq:objective of S}
\end{align}
where ``$\propto$'' means ``proportional to'', and
\if\thisismainpaper1
\begin{align}
\nu := {\sqrt{L^2\log(2/\gamma)/2} \over \lambda_l + \lambda_\eta LC},\quad
\tilde{\ubf} :=\left({u_j\over n}\right)_{j: \xbf_j\in S}. \label{eq:nu}
\end{align}
\else
\begin{align}
\nu &:= {\sqrt{L^2\log(2/\gamma)/2} \over \lambda_l + \lambda_\eta LC}, \label{eq:nu}\\
\tilde{\ubf} &:=\left({u_j\over n}\right)_{j: \xbf_j\in S}. \label{eq:tilde_u}
\end{align}
\fi
Since the coreset $S$ induces a clustering of the original sample set $P$, with each cluster containing all the samples represented by the same point in $S$, $\delta$ is essentially the maximum radius of each cluster, and $\tilde{\ubf}$ is the distribution of samples across the clusters. As $\sum_{j: \xbf_j\in S}(u_j/n) = 1$,  Jensen's inequality implies that $\|\tilde{\ubf}\|\geq 1/\sqrt{|S|}$, with the minimum achieved at $u_j = n/|S|$ ($\forall j: \xbf_j\in S$). 
Thus, to minimize \eqref{eq:objective of S}, the coreset $S$ should \emph{cover all the samples with the minimum radius, while balancing the number of samples represented by each coreset point}. 

\emph{Remark:}
Note that while it is natural to represent each sample in $P$ by the nearest point in $S$, i.e., $j_i:=\argmin_{j: \xbf_j\in S}\|\xbf_i-\xbf_j\|$, this is not required. Which subset of $P$ is represented by each point in $S$ is a decision variable in coreset construction, in addition to the set $S$ itself, and can be used to trade off between $\delta$ and $\|\tilde{\ubf}\|$. \looseness=-1

\section{Algorithm Design}\label{sec:Algorithm Design}

We now design a sampling algorithm based on coreset construction, so that only the samples selected into the coreset are collected. The difference from classical coreset construction is that we only have noisy predictions of the unlabeled samples. While given a forecasting model that can return the predicted values of the next batch of candidate samples, we can perform coreset construction on this predicted set, 
the key question is how to adjust for prediction errors, which will be addressed below. 
\looseness=0

\subsection{Sampling by Predictive Coreset Construction}\label{subsec:Sampling by Predictive Coreset Construction}

\begin{algorithm}[tb]
\small
\SetKwInOut{Input}{input}\SetKwInOut{Output}{output}
\Input{An $n$-sample forecasting model, coverage radii $\delta_0$ and $\delta_1$, weight bound $\kappa$, initial set of collected samples $S^{(0)}$}
\Output{Overall set of collected samples and their weights $\tilde{S}$}
Set initial weights $u_j\leftarrow 1$, $\forall \xbf_j\in S^{(0)}$\;
\For{prediction window $T=1,2,\ldots$ each containing $n$ sampling epochs\label{sampling:1}}
{
Predict the next $n$ samples $\hat{P}\leftarrow \{\hat{\xbf}_i\}_{i=1}^n$\; \label{sampling:2}
$Q\leftarrow \hat{P}$\tcp*{samples not covered by $S^{(0)}$} \label{sampling:Q - 1}
\ForEach{$\hat{\xbf}_i\in \hat{P}$}
{\If{$\exists \xbf_j\in S^{(0)}$ with $\|\xbf_j-\hat{\xbf}_i\|\leq \delta_0$ \& $u_j<\kappa$}
{$Q\leftarrow Q\setminus \{\hat{\xbf}_i\}$\;
$u_j\leftarrow u_j +1$\;}
}\label{sampling:Q - 2}
Select the minimum subset $\hat{S}^{(1)}\subseteq \hat{P}$ that forms a $\delta_1$-cover of $Q$ such that the number of points $u_j$ covered by each $\hat{\xbf}_j\in \hat{S}^{(1)}$ satisfies $u_j\leq \kappa$\; \label{sampling:4}
Collect the samples in $S^{(1)}\leftarrow \{\xbf_j:\: \hat{\xbf}_j\in \hat{S}^{(1)}\}$\; \label{sampling:5}
Merge the collected samples by $S^{(0)}\leftarrow S^{(0)}\cup S^{(1)}$\; \label{sampling:6}
}
Return $\tilde{S}\leftarrow \{(\xbf_j,\: u_j)\}_{\xbf_j\in S^{(0)}}$\; \label{sampling:7}
\caption{Predictive Coreset Construction}
\label{Alg:predictive coreset}
\vspace{-.05em}
\end{algorithm}

We propose a sampling algorithm based on \emph{predictive coreset construction} as shown in Algorithm~\ref{Alg:predictive coreset}. Our algorithm works sequentially on each prediction window of a given size $n$, specified by the forecasting model. Given the predicted samples $\hat{P}$ in the current window (line~\ref{sampling:2}), the algorithm first checks whether each predicted sample $\hat{\xbf}_i\in \hat{P}$ is covered by an existing sample in $S^{(0)}$ within a given radius $\delta_0$ (lines~\ref{sampling:Q - 1}--\ref{sampling:Q - 2}), and then tries to select a minimum subset $\hat{S}^{(1)}$ of $\hat{P}$ to cover the remaining samples within another given radius $\delta_1$ (line~\ref{sampling:4}). 
Here, we say that a set of points $B$ is a \emph{$\delta$-cover} of another set of points $A$ if $\forall \bm{a}\in A$, $\exists \bm{b}\in B$ such that $\|\bm{a}-\bm{b}\|\leq \delta$. 
Each $\hat{\xbf}_j$ in $\hat{S}^{(1)}$ represents a sampling epoch (e.g., a superframe) at which a new sample $\xbf_j$ will be collected (line~\ref{sampling:5}). These new samples are then merged with the existing samples in $S^{(0)}$ (line~\ref{sampling:6}) in preparation for the next prediction window. 
During this process, the algorithm keeps track of the weight $u_j$ for each collected sample $\xbf_j$ to denote the number of candidate samples represented by $\xbf_j$, where we allow a sample $\xbf_i$ to be represented by another sample $\xbf_j$ if the predicted value $\hat{\xbf}_i$ is covered by $\xbf_j$ within distance $\delta_0$ in the case of $\xbf_j\in S^{(0)}$ or $\hat{\xbf}_j$ within distance $\delta_1$ in the case of $\xbf_j\in S^{(1)}$. The algorithm makes sure that $u_j\leq \kappa$ for a given integral parameter $\kappa\geq 1$. 
The collected set of weighted unlabeled samples $\tilde{S}$ (where the weight of each sample denotes its multiplicity) will then be labeled and used to train the target model. \looseness=-1

\emph{Remark:} In cases where the data sources can be directly programmed to implement sampling, Algorithm~\ref{Alg:predictive coreset} remains applicable, except that line~\ref{sampling:2} is replaced by directly collecting the next window of samples $P=\{\xbf_i\}_{i=1}^n$ and line~\ref{sampling:5} is skipped (as line~\ref{sampling:4} has selected the samples $S^{(1)}$ to keep). Our performance analysis in Theorem~\ref{thm:delta-cover guarantee} holds trivially in this case with $\delta_0=\delta_1=\delta$.

\subsection{Complexity Analysis} 

Given the predicted samples, the complexity of Algorithm~\ref{Alg:predictive coreset} for each prediction window is dominated by the coverage of the predicted samples using the existing samples in $S^{(0)}$ (lines~\ref{sampling:Q - 1}--\ref{sampling:Q - 2}) and the coverage of the remaining predicted samples using the predicted values of new samples in $S^{(1)}$ (line~\ref{sampling:4}). Lines~\ref{sampling:Q - 1}--\ref{sampling:Q - 2} can be completed in $O(n |S^{(0)}|)$ time by a linear search in $S^{(0)}$ for each $\hat{\xbf}_i\in \hat{P}$ (assuming $\|\xbf_j-\hat{\xbf}_i\|$ can be evaluated in $O(1)$ time). Line~\ref{sampling:4} is a combinatorial optimization problem that can be formulated as an integer linear program (ILP) as follows. 

Let $\alpha_j$ indicate whether $\hat{\xbf}_j\in \hat{S}^{(1)}$, and $\beta_{ij}$ indicate whether $\hat{\xbf}_i\in Q$ is covered by $\hat{\xbf}_j\in \hat{P}$. Then line~\ref{sampling:4} aims at solving:
\begin{subequations}\label{eq:optimization for sampling4}
\begin{align}
    \min \quad& \sum_{j=1}^n \alpha_j \\
    \mbox{s.t. } & \sum_{i: \hat{\xbf}_i\in Q} \beta_{ij} \leq \kappa \alpha_j,~~~\forall j=1,\ldots,n, \label{cons:weight} \\
    & \sum_{j=1}^n \beta_{ij} = 1, ~~~\forall i:\: \hat{\xbf}_i\in Q, \label{cons:coverage} \\
    & \beta_{ij} = 0,~~~\forall (i,j):\: \|\hat{\xbf}_i-\hat{\xbf}_j\|>\delta_1, \label{cons:radius} \\
    & \alpha_j,\: \beta_{ij}\in \{0,1\},~~~\forall j=1,\ldots,n,\:  i:\: \hat{\xbf}_i\in Q.
\end{align}
\end{subequations}
This is a generalization of the \emph{minimum set cover (MSC)} problem, which aims to select a minimum number of sets from the collection $\left\{\{\hat{\xbf}\in Q:\: \|\hat{\xbf}-\hat{\xbf}_j\|\leq \delta_1\}\right\}_{j=1}^n$ to cover $Q$, with the additional constraint that the number of points in $Q$ covered by each set satisfies a given upper bound. As MSC is NP-hard~\cite{Korte12book}, \eqref{eq:optimization for sampling4} is NP-hard. However, as the prediction window size $n$ is usually small, in practice line~\ref{sampling:4} can often be solved optimally. 
In \if\thisismainpaper1
the appendix of \cite{Chiu24:report}, 
\else
Appendix~\ref{appendix:Implementation of Algorithm}, 
\fi
we describe a partially brute-force method to do so in $O(n^3 2^n)$ time. 
For larger $n$, line~\ref{sampling:4} can be approximately solved by any heuristic for ILP (e.g., LP relaxation with randomized rounding). Thus, each iteration in Algorithm~\ref{Alg:predictive coreset} has a complexity of $O(n|S^{(0)}|+f(n)) = O(n|S^{(0)}|)$, assuming $n=O(1)$ (where $f(n)$ is the time complexity of line~\ref{sampling:4}). As $|S^{(0)}|$ after processing $t$ windows is bounded by $nt$, the total complexity for processing $T$ windows is $O(\sum_{t=1}^T n^2 t) = O(n^2 T^2)$, which is \emph{quadratic} in the total number of candidate samples $nT$.

Meanwhile, the complexity is substantially lower for an important special case. Specifically, if the target model is a deep neural network, then the Lipschitz constant $\lambda_l$ for its loss function can be very large (e.g., exponential in the number of layers~\cite{Sener18ICLR}). In this case, the coefficient $\nu$ in \eqref{eq:nu} will be very small, which means that the error bound \eqref{eq:approx error bound} depends on the coreset $S$ mainly through the coverage radius $\delta$, and thus we can remove the constraint on the number of candidate samples represented by each point in $S$ by 
setting $\kappa=\infty$. This reduces lines~\ref{sampling:Q - 1}--\ref{sampling:Q - 2} to the test for each $\hat{\xbf}_i\in \hat{P}$ to fall in the $\delta_0$-cover of the set of existing samples $S^{(0)}$, and line~\ref{sampling:4} to the selection of the minimum subset of $\hat{P}$ to form a $\delta_1$-cover of the predicted samples in $Q$. 
The former can be solved by performing a \emph{nearest neighbor search (NNS)} for each $\hat{\xbf}_i$ in $S^{(0)}$ and comparing the distance to the nearest neighbor with $\delta_0$. Using an appropriate data structure such as {$k$-$d$ tree}, this step can be implemented with an average complexity of $O(n\log{|S^{(0)}|})$~\cite{Moore04turotial}. 
The latter is an instance of the MSC problem, where we want to select a minimum number of sets from the collection $\left\{\{\hat{\xbf}\in Q:\: \|\hat{\xbf}-\hat{\xbf}_j\|\leq \delta_1\}\right\}_{j=1}^n$ to cover $Q$. 
For small $n$, this is 
solvable at a complexity of $O(2^n)$ (via brute force). For larger $n$, we can apply any approximate MSC algorithm such as the greedy algorithm (i.e., iteratively selecting the set containing the largest number of uncovered elements), which achieves an approximation ratio of $O(\log{|Q|})= O(\log{n})$ that is known to be near-optimal \cite{Dinur14STOC}. Thus, in the special case of $\kappa=\infty$, each iteration in Algorithm~\ref{Alg:predictive coreset} has a complexity of $O(n\log{|S^{(0)}|} + 2^n)=O(n\log{|S^{(0)}|})$ if $n=O(1)$. The total complexity over $T$ windows is then $O(\sum_{t=1}^T n\log(nt)) = O(nT\log{T})$, which is \emph{nearly linear} in the total number of candidate samples $nT$.

\subsection{Performance Analysis}\label{subsubsec:Performance Analysis}

Based on the result in \eqref{eq:objective of S}, it is desirable that the set of collected samples (i.e., coreset) can cover the set of candidate samples (i.e., the complete dataset) with a small radius $\delta$ and a small norm  $\|\tilde{\ubf}\|$.  
We will show that under the assumption of i.i.d. Gaussian prediction errors as in assumption~\eqref{assumption:prediction error}, Algorithm~\ref{Alg:predictive coreset} can be configured to make sure that: (i) the coverage radius will be bounded by a given $\delta$ with a guaranteed confidence, and (ii) the norm  $\|\tilde{\ubf}\|$ will decay with the number of collected samples $|S|$ at the rate of $1/\sqrt{|S|}$. 

\begin{theorem}\label{thm:delta-cover guarantee}
Consider the $T$-th prediction window for any $T\geq 1$, where $P:=\{\xbf_i\}_{i=1}^n$ is the set of candidate samples in this window, $S^{(0)}$ is the set of samples collected before sampling from this window, and $S^{(1)}$ is the set of samples collected from this window by Algorithm~\ref{Alg:predictive coreset}. Then $S:= S^{(0)}\cup S^{(1)}$ and the corresponding weights $\{u_j\}_{j: \xbf_j\in S}$ satisfy:
\begin{enumerate}
    \item $\|\tilde{\ubf}\|\leq {\kappa\over \sqrt{|S|}}$ for $\tilde{\ubf}:=(u_j/(nT+s_0))_{j: \xbf_j\in S}$, where $s_0$ is the number of initial samples; 
    \item under assumption~\eqref{assumption:prediction error}, setting 
\begin{align}
    \delta_0 &:= \delta-\sqrt{\sigma_n^2 F^{-1}\left((1-\epsilon)^{1/n};d\right)}, \label{eq:delta_0}\\
    \delta_1 &:= \delta-\sqrt{2\sigma_n^2 F^{-1}\left((1-\epsilon)^{1/n};d\right)} \label{eq:delta_1}
\end{align} 
ensures that with probability at least $1-\epsilon$, each $\xbf_i\in P$ and its representation $\xbf_{j_i}\in S$ are within distance $\delta$, where $d$ is the dimensionality of the input space and $F^{-1}(\cdot;d)$ is the inverse of the \emph{cumulative distribution function (CDF)} of the chi-squared distribution with $d$ degrees of freedom. 
\end{enumerate}
\end{theorem}

\emph{Remark:} The shrinkage of coverage radius in \eqref{eq:delta_0}--\eqref{eq:delta_1} is used to compensate for prediction error. Thus, more shrinkage is needed if the prediction error $\sigma_n^2$ increases or the confidence level $1-\epsilon$ increases. As the coverage radius should be non-negative, we have
\begin{align}    
\delta \geq \sqrt{2\sigma_n^2 F^{-1}\left((1-\epsilon)^{1/n};d\right)},
\end{align}
which sets a lower bound on the approximation error that Algorithm~\ref{Alg:predictive coreset} can be guaranteed to achieve with probability $\geq 1-\epsilon$.\looseness=-1 

In the special case of $n=1$ (i.e., predicting one sample at a time) and $\kappa=\infty$, Algorithm~\ref{Alg:predictive coreset} is reduced to a simple threshold policy: given the next prediction $\hat{P}=\{\hat{\xbf}_1\}$, lines~\ref{sampling:Q - 1}--\ref{sampling:4} are reduced to:
\begin{align}\label{eq:threshold policy}
    \hat{S}^{(1)} = \left\{\begin{array}{ll}
    \{\hat{\xbf}_1\} & \mbox{if }\min_{\xbf\in S^{(0)}}\|\hat{\xbf}_1-\xbf\| > \delta_0,\\
    \emptyset & \mbox{o.w.,}
    \end{array}\right.
\end{align}
i.e., collecting a new sample if and only if its predicted value is not in the $\delta_0$-cover of the existing sample set. Setting $n=1$ and $\kappa=\infty$ in Theorem~\ref{thm:delta-cover guarantee} leads to the following result. 

\begin{corollary}\label{cor:threshold policy}
Under assumption \eqref{assumption:prediction error}, $n=1$, and $\kappa=\infty$, the sample set $S$ collected by the threshold policy specified in \eqref{eq:threshold policy} after processing a predicted sample $\hat{\xbf}$ is a $\delta$-cover of its true value $\xbf$ with probability at least $1-\epsilon$ if \looseness=-1
\begin{align}
    \delta_0 = \delta-\sqrt{\sigma_1^2 F^{-1}(1-\epsilon; d)},
\end{align}
where $\sigma_1^2$ is the MSE in predicting each attribute of $\xbf$.
\end{corollary}

\emph{Remark:} Theorem~\ref{thm:delta-cover guarantee} and Corollary~\ref{cor:threshold policy} provide theoretically-justified ways to set the input parameters $\delta_0$, $\delta_1$, and $\kappa$ of Algorithm~\ref{Alg:predictive coreset}. They in turn depend on other parameters such as $\delta$ that is directly related to the approximation of the training loss function through Theorem~\ref{thm:approx error bound}, and will be treated as hyperparameters of the learning task to be tuned to achieve the desired tradeoff between the data curation cost and the quality of the target model. 



\section{Performance Evaluation}\label{sec:Performance Evaluation}

We evaluate our solution in practical settings based on both public health monitoring data and our own prototype implementation. 

\subsection{Data-driven Simulation}\label{subsec:Data-driven Simulation}

For reproducibility, we first conduct a data-driven simulation based on a publicly available dataset. 

\subsubsection{Evaluation Setting}\label{subsubsec:Evaluation Setting - Simulation}

\textbf{Dataset:}
We use the FitBit Fitness Tracker Data \cite{FitBit} generated by survey respondents via Amazon Mechanical Turk. This dataset contains various physiological metrics including heart rate, step count, calorie consumption, and intensity level, among others. In particular, the intensity level is computed using a proprietary algorithm developed by FitBit, which assigns a value between 0 and 3 to indicate the intensity of physical activity, with 0 indicating the lowest level and 3 the highest level. 
While the raw data are collected at variable intervals, we convert them to evenly-spaced time series by taking the average over one-minute intervals. 
The dataset contains data from different users over non-contiguous time periods, from which we extract a subset of {43,920} data points (each covering one minute) from a single user with contiguous data points in each five-minute window. This window size is chosen to yield a good tradeoff between classification accuracy and the number of windows with contiguous data. We partition the extracted dataset in half, using the first half to train a forecasting model {(with $80\%$ of data for training and $20\%$ for validation)} and the other half to train and test a target model for intensity level classification {(with $90\%$ of data for training and $10\%$ for testing)}.\looseness=-1

We treat heart rate and step count as the input of the target model and intensity level as the output. To evaluate the impact of prediction accuracy, we consider three cases:\begin{enumerate}
\item[1.] both heart rate and step count are measured by body sensors and need to be predicted by the coordinator (subject to prediction errors); 
\item[2.] heart rate is measured by a body sensor, but step count is measured by a local sensor at the coordinator and therefore not subject to prediction errors; 
\item[3.] both heart rate and step count are measured locally at the coordinator\footnote{Equivalently, this models the case when sampling is performed at the FitBit device.} and not subject to prediction errors.
\end{enumerate}
Recall that sampling decisions are supposed to be made by the coordinator.  
Besides evaluating the impact of diminishing prediction errors, these cases also represent the different application scenarios of: sampling from remote data sources without local data (Case~1), sampling from remote data sources with some local data (Case~2), and sampling from local data (Case~3). 


\textbf{Forecasting model:}
For time series forecasting, we utilize a two-layer Long Short-Term Memory (LSTM) model with a hidden dimension of 64. The input comprises the (collected or predicted) heart rate and step count data within a five-minute window, and the output is the prediction for the next five-minute window, both at one-minute granularity. In Case~1, both heart rate and step count are predicted; in Case~2, only heart rate is predicted (no prediction needed in Case~3). To train the model, we employ the squared error loss function and the Adam algorithm \cite{Kingma15ICLR} with a learning rate of 0.01. Note that our focus is \emph{not} on time series forecasting and other forecasting models can be used as well. 


\textbf{Target model:} 
To classify the intensity level, we employ a fully-connected feedforward neural network with two hidden layers of rectified linear units (ReLUs), an output layer of softmax unit, and batch normalization. The model takes as input the heart rate and step count data over a five-minute window, and outputs the estimated intensity level in the last minute as the label. We train the model on data selected by an active learning algorithm (after labeling) using the Adam algorithm \cite{Kingma15ICLR} based on cross-entropy loss and a learning rate of 0.01. 
As sampling can lead to a small training set and hence increased risk of overfitting, we incorporate dropout with a probability of 0.2 into the second hidden layer. 
Moreover, we vary the target model size to account for varying amounts of training data: if the training set contains more than 500 data points, we will use a target model with 16 neurons in the first layer and 8 neurons in the second layer, trained over {200} epochs; if the training set contains fewer than 500 data points, we will use a  target model with 4 neurons in each of the hidden layers, trained over 500 epochs. We find the above target models to perform sufficiently well (with $99\%$ accuracy over a wide range of sampling ratios). However, they are just examples to evaluate the active learning algorithms and other target models can also be used. \looseness=-1 

\textbf{Benchmarks:}
For active learning without the ability to query labels in real time, only our solution and the solution from Sener et al. \cite{Sener18ICLR} are applicable. 
We treat each minute as a sampling epoch. Since the forecasting model predicts for five minutes at a time, it implies $n=5$ in Algorithm~\ref{Alg:predictive coreset}. Specifically, based on the next five predicted data points, Algorithm~\ref{Alg:predictive coreset} selects a subset of the predicted data points into $\hat{S}^{(1)}$, 
and only collects the data at the one-minute intervals represented in $\hat{S}^{(1)}$. 
After the data collection, we label each sampled data point with its corresponding intensity level, and concatenate it with the previous four data points to create one training sample for the target model. If some of the previous four data points are not sampled, we will replace them by their predicted values.  
\rev{In WBAN, the critical devices to preserve energy for are the body sensors, as the coordinator can be easily recharged. Since the sensors are passively polled by the coordinator when it decides to collect a sample, the energy consumption at sensors is proportional to the sampling ratio.}
We set $\kappa=\infty$ and tune $(\delta_0, \delta_1)$ to achieve various sampling ratios for Algorithm~\ref{Alg:predictive coreset}. We then set the same sampling ratios for \cite{Sener18ICLR} (``Sener'') for comparison. As a baseline, we also evaluate the method of uniformly selecting samples at random (``Random'') and show its average results over 10 Monte Carlo runs. 
We have also evaluated different values of $\kappa$ and shown the results in \if\thisismainpaper1
\cite{Chiu24:report}. 
\else
Appendix~\ref{appendix:Additional Evaluation}. 
\fi
\looseness=-1

\subsubsection{Evaluation Results}

\textbf{Results of forecasting:}
Fig.~\ref{fig:forecasting} shows the accuracy of the adopted forecasting model in each of the simulated cases, based on true sensor measurements as the input. Note that only the heart rate data need to be predicted in Case~2, and no prediction is needed in Case~3. 
This result shows that with a properly-selected forecasting model, one can meaningfully predict health monitoring data in a sufficiently near future (of five minutes here), although with non-negligible prediction errors. 
In our case, the prediction error is less for heart rate than step count: the normalized root mean squared error (NRMSE) for heart rate prediction is {$0.0855$} in Case~1 and {$0.0839$} in Case~2, and the NRMSE for step count prediction (in Case~1) is {$1.9339$}.  
Note that this is just an initial test of predictability. During active learning, we will replace any missing data in the forecasting model's input by their predicted values, and evaluate the consequence of error propagation through its impact on active learning and target model training. \looseness=-1

\begin{figure}[t!]
\vspace{0em}
\begin{minipage}{.495\linewidth}
\centerline{\mbox{\includegraphics[width=1.0\linewidth,height=\figheight]
{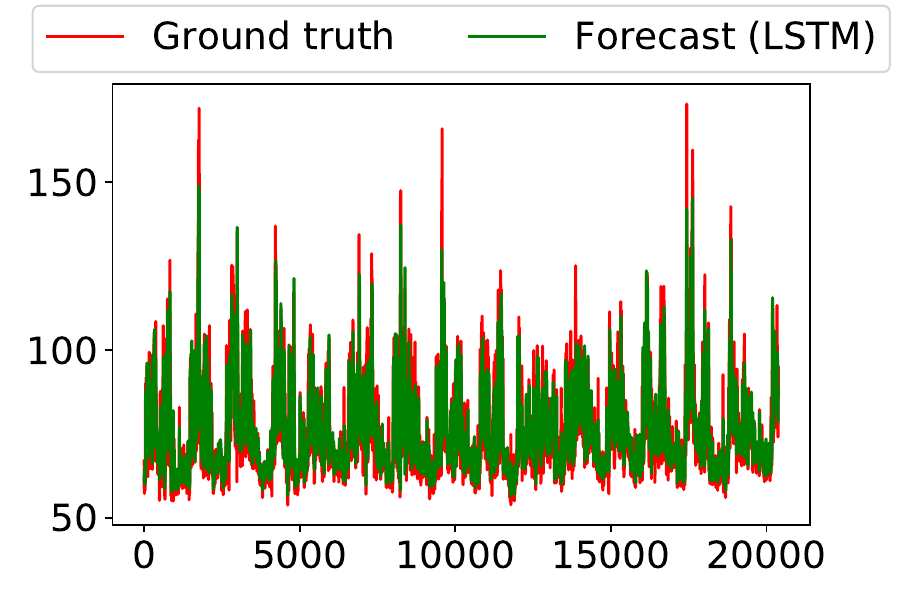}}}
\vspace{-.5em}
\centerline{(a) Case 1: heart rate}
\end{minipage}\hfill
\begin{minipage}{.495\linewidth}
    \centerline{\mbox{\includegraphics[width=1.0\linewidth,height=\figheight]
{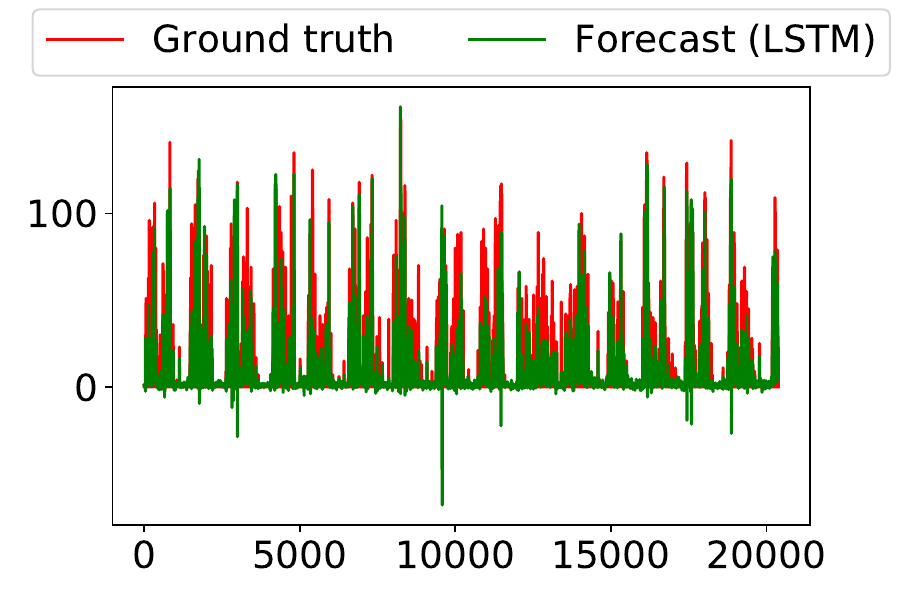}}}
\vspace{-.5em}
\centerline{(b) Case 1: step count}
\end{minipage}\\
\vspace{-.0em}
\centerline{\begin{minipage}{.495\linewidth}
\centerline{\mbox{\includegraphics[width=1.0\linewidth,height=\figheight]
{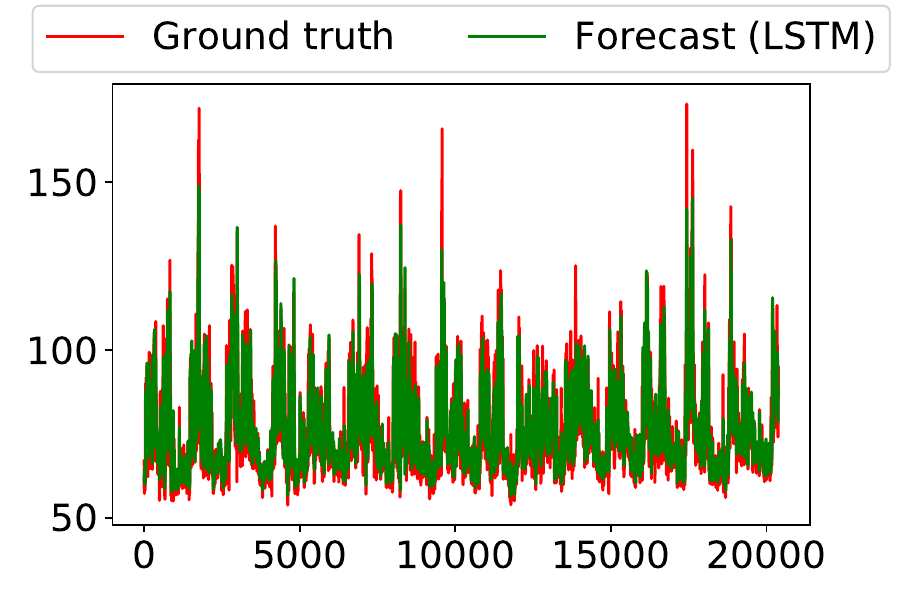}}}
\vspace{-.5em}
\centerline{(c) Case 2: heart rate}
\end{minipage}}
\vspace{-1em}
\caption{\small Results of time series forecasting on public data. }\label{fig:forecasting}
\vspace{-.5em}
\end{figure}

\textbf{Results of sampling:}
Fig.~\ref{fig:sample distribution: case 1: n=5}--\ref{fig:sample distribution: case 3: n=5} show the distribution of the selected samples across the intensity levels (in $\log$ scale). Before sampling, the dataset is highly skewed with {$75.6\%$, $21.7\%$, $1.3\%$, $1.4\%$ samples at intensity level $0$, $1$, $2$, $3$, respectively}. Such skewness is directly inherited by random sampling, and also causes significant skewness for Sener's method \cite{Sener18ICLR} as it selects the same number of samples from every prediction window. In contrast, our algorithm notably reduces the skewness by intentionally selecting the samples sufficiently distinct from each other, resulting in a more balanced dataset, even if it does not know the labels.  
\rev{While our algorithm has shown a class balancing effect, it is not to be confused with 
\if\thisismainpaper1
class balancing techniques~\cite{Cui19CVPR}, 
\else
class balancing techniques~\cite{Cui19CVPR}, 
\fi
which require the samples to be labeled. In contrast, our algorithm works on unlabeled samples, and its class balancing effect is a byproduct of our coreset-based approach.}

\begin{figure}[t!]
\vspace{-.0em}
\begin{minipage}{.495\linewidth}
\centerline{\mbox{\includegraphics[width=1.0\linewidth,height=\figheight]
{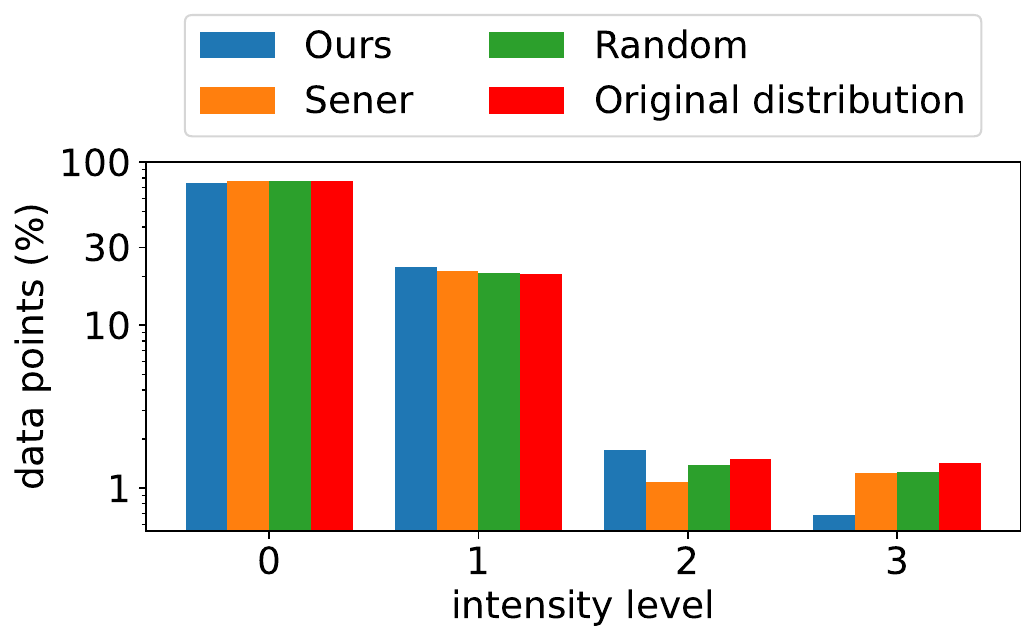}}}
\vspace{-.5em}
\centerline{(a) $2\%$ sampling}
\end{minipage}\hfill
\begin{minipage}{.495\linewidth}
\centerline{\mbox{\includegraphics[width=1.0\linewidth,height=\figheight]
{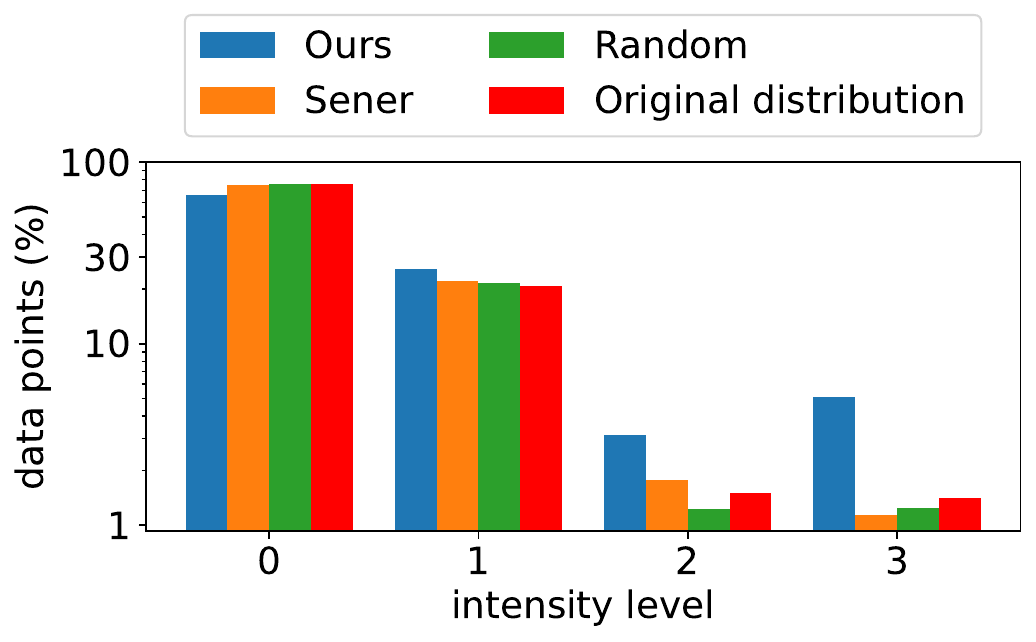}}}
\vspace{-.5em}
\centerline{(b) $3\%$ sampling}
\end{minipage}
\vspace{-1em}
\caption{\small { Distribution of selected samples in Case~1.}}\label{fig:sample distribution: case 1: n=5}
\end{figure}

\begin{figure}[t!]
\vspace{.0em}
\begin{minipage}{.495\linewidth}
\centerline{\mbox{\includegraphics[width=1.0\linewidth,height=\figheight]
{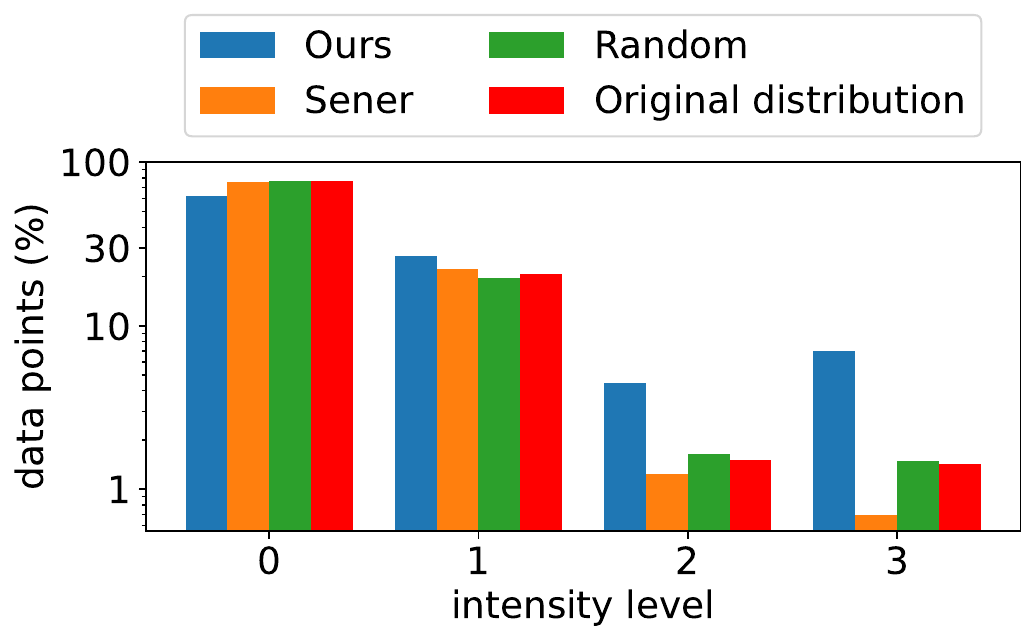}}}
\vspace{-.5em}
\centerline{(a) $1\%$ sampling}
\end{minipage}\hfill
\begin{minipage}{.495\linewidth}
\centerline{\mbox{\includegraphics[width=1.0\linewidth,height=\figheight]
{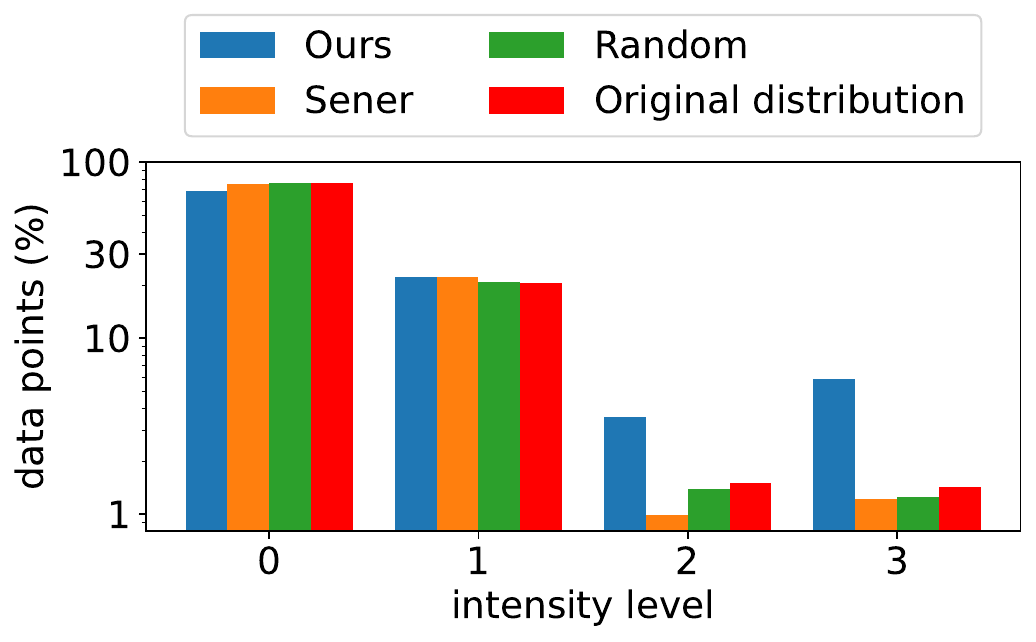}}}
\vspace{-.5em}
\centerline{(b) $2\%$ sampling}
\end{minipage}
\vspace{-1em}
\caption{\small { Distribution of selected samples in Case~2. }}\label{fig:sample distribution: case 2: n=5}
\end{figure}

\begin{figure}[t!]
\vspace{0em}
\begin{minipage}{.495\linewidth}
\centerline{\mbox{\includegraphics[width=1.0\linewidth,height=\figheight]
{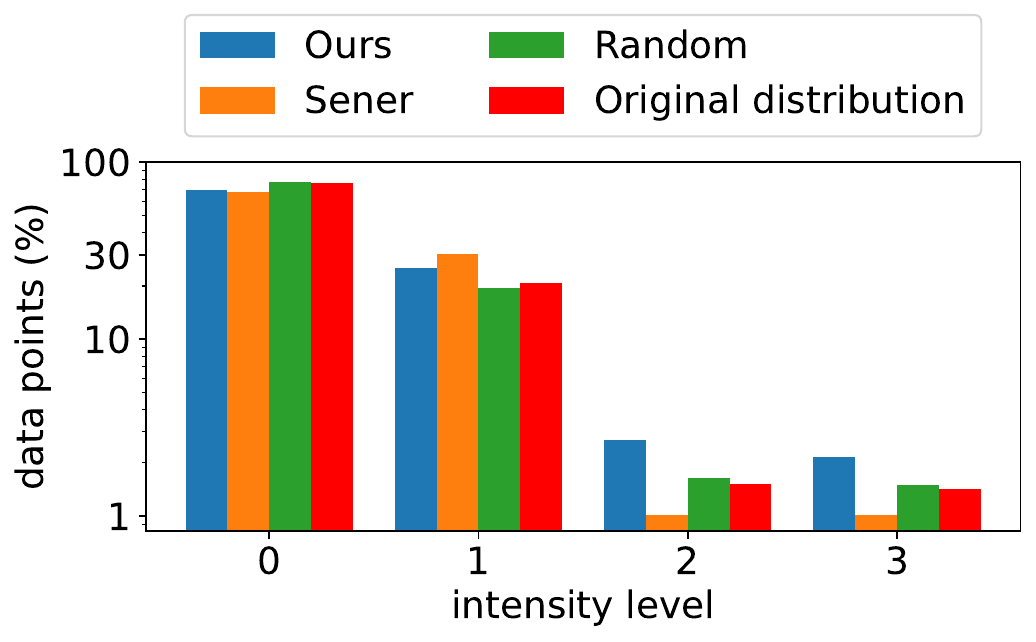}}}
\vspace{-.5em}
\centerline{(a) $1\%$ sampling}
\end{minipage}\hfill
\begin{minipage}{.495\linewidth}
\centerline{\mbox{\includegraphics[width=1.0\linewidth,height=\figheight]
{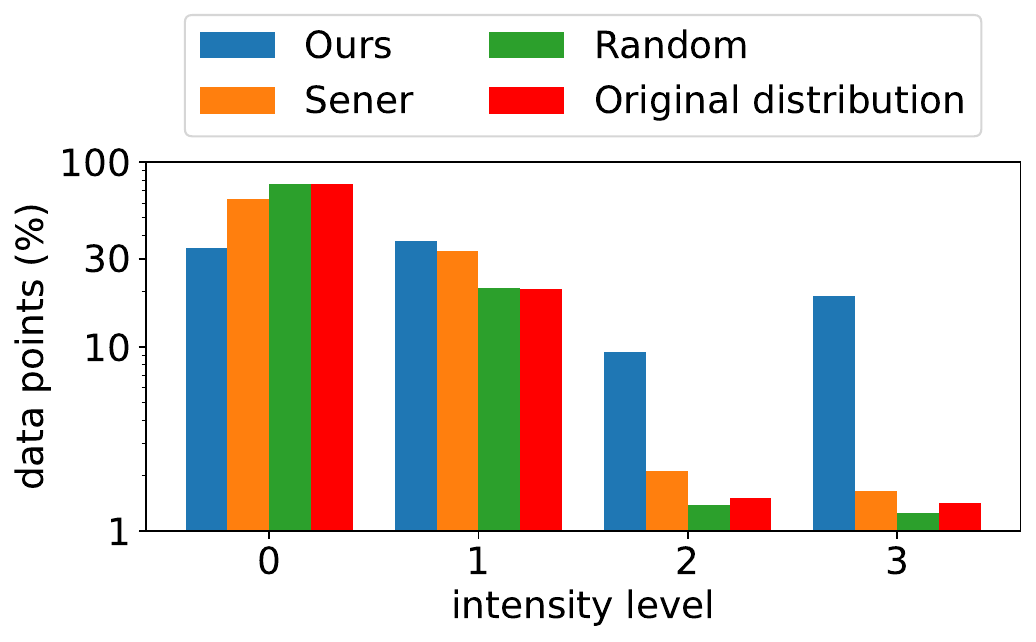}}}
\vspace{-.5em}
\centerline{(b) $2\%$ sampling}
\end{minipage}
\vspace{-1em}
\caption{\small { Distribution of selected samples in Case~3. } }\label{fig:sample distribution: case 3: n=5}
\end{figure}

\textbf{Quality of trained model:}
Fig.~\ref{fig:target model: case 1: n=5}--\ref{fig:target model: case 3: n=5} show the performance of the target model trained on the samples selected by each active learning algorithm (after labeling). We evaluate the quality of the target model by two metrics: the classification accuracy over all the testing samples, and the F1 score in correctly identifying the samples of each intensity level, averaged across all the levels. We have used the average F1 score to choose the initial model parameters among 10 different initial values. 
We consider the average F1 score because the testing set is also skewed with more than {$75\%$} of samples at intensity level 0, and thus a classifier can achieve good accuracy without being able to identify all the intensity levels. Indeed, our results confirm that our algorithm can produce a target model with a much higher average F1 score than the benchmarks (particularly at low sampling ratios) even though the accuracy is similar, thanks to its ability of balancing the sample distribution as shown in Fig.~\ref{fig:sample distribution: case 1: n=5}--\ref{fig:sample distribution: case 3: n=5}. In particular, our algorithm can achieve almost the same performance as training the target model on the full dataset, by only collecting and labeling $1$--$2\%$ of the data. Meanwhile, comparing Fig.~\ref{fig:target model: case 1: n=5}, Fig.~\ref{fig:target model: case 2: n=5}, and Fig.~\ref{fig:target model: case 3: n=5} shows a slight improvement in the target model trained by our algorithm when the prediction accuracy is improved. 

\begin{figure}[t!]
\vspace{-.0em}
\begin{minipage}{.495\linewidth}
\centerline{\mbox{\includegraphics[width=1.0\linewidth,height=\figheight]
{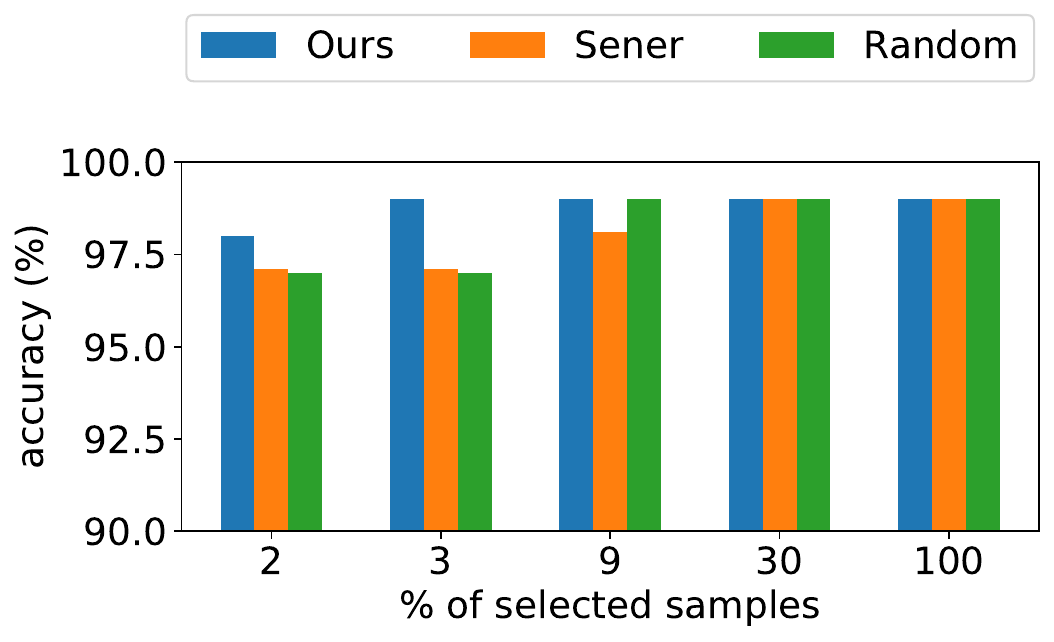}}}
\vspace{-.5em}
\centerline{(a) accuracy}
\end{minipage}\hfill
\begin{minipage}{.495\linewidth}
\centerline{\mbox{\includegraphics[width=1.0\linewidth,height=\figheight]
{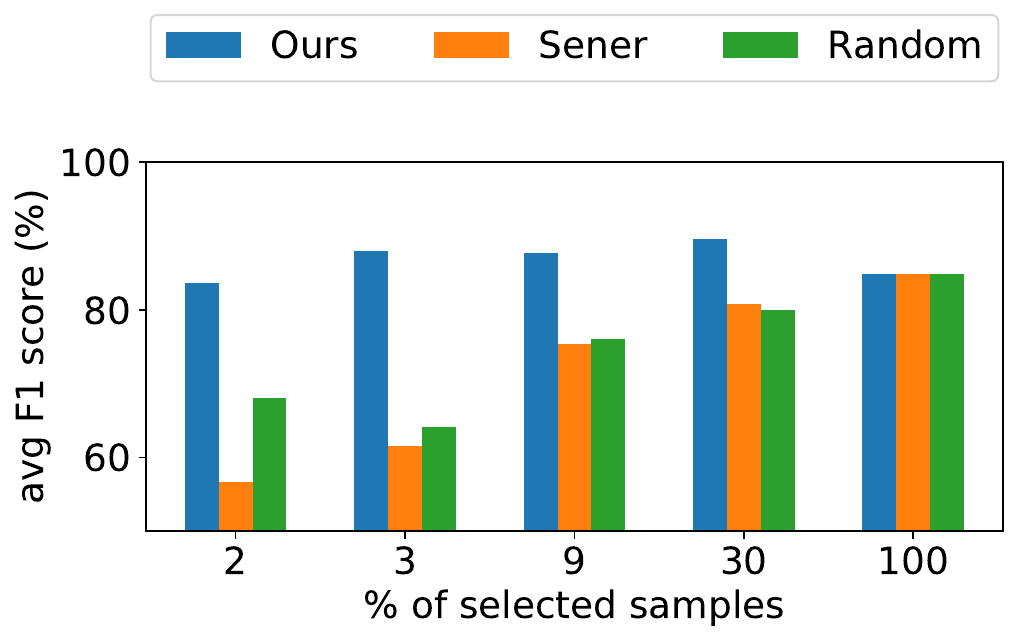}}}
\vspace{-.5em}
\centerline{(b) average F1 score}
\end{minipage}
\vspace{-1em}
\caption{\small { Quality of trained target model in Case~1. 
}}\label{fig:target model: case 1: n=5}
\end{figure}

\begin{figure}[t!]
\vspace{-.0em}
\begin{minipage}{.495\linewidth}
\centerline{\mbox{\includegraphics[width=1.0\linewidth,height=\figheight]
{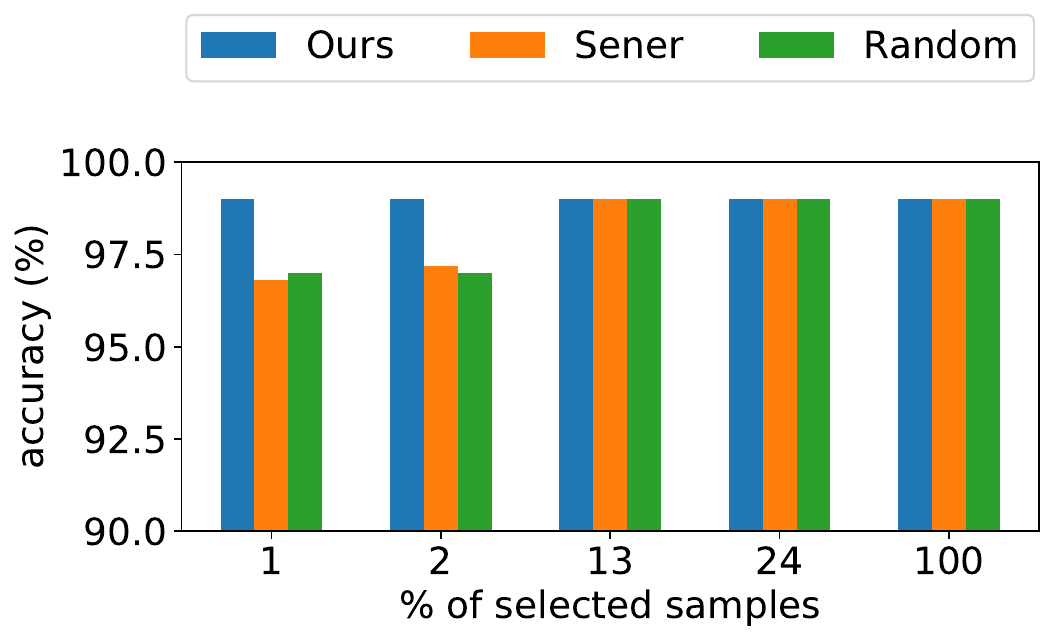}}}
\vspace{-.5em}
\centerline{(a) accuracy}
\end{minipage}\hfill
\begin{minipage}{.495\linewidth}
\centerline{\mbox{\includegraphics[width=1.0\linewidth,height=\figheight]
{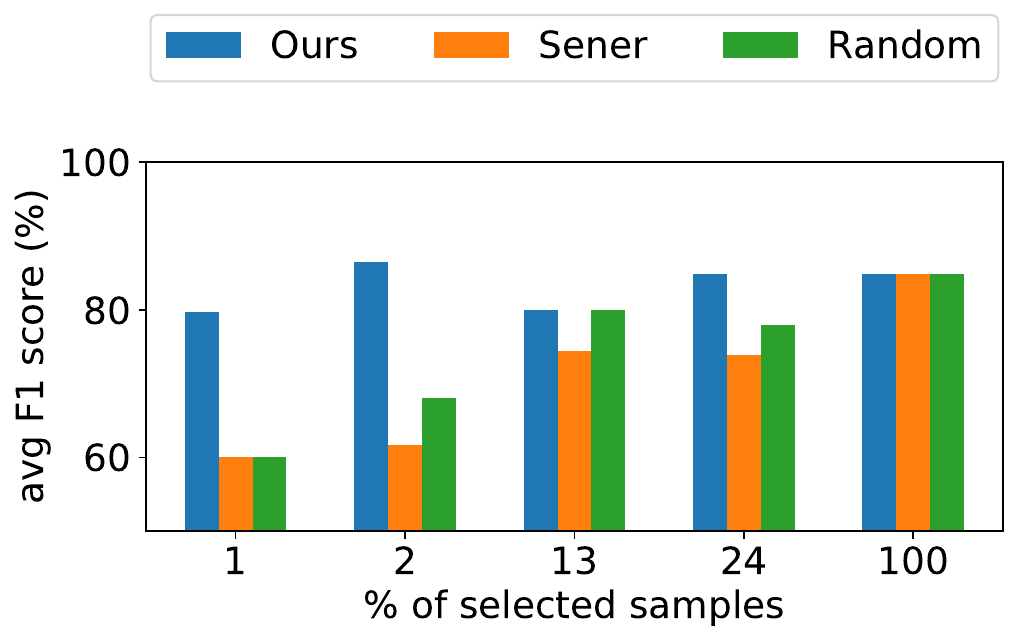}}}
\vspace{-.5em}
\centerline{(b) average F1 score}
\end{minipage}
\vspace{-1em}
\caption{\small { Quality of trained target model in Case~2. }}\label{fig:target model: case 2: n=5}
\end{figure}

\begin{figure}[t!]
\vspace{-.0em}
\begin{minipage}{.495\linewidth}
\centerline{\mbox{\includegraphics[width=1.0\linewidth,height=\figheight]
{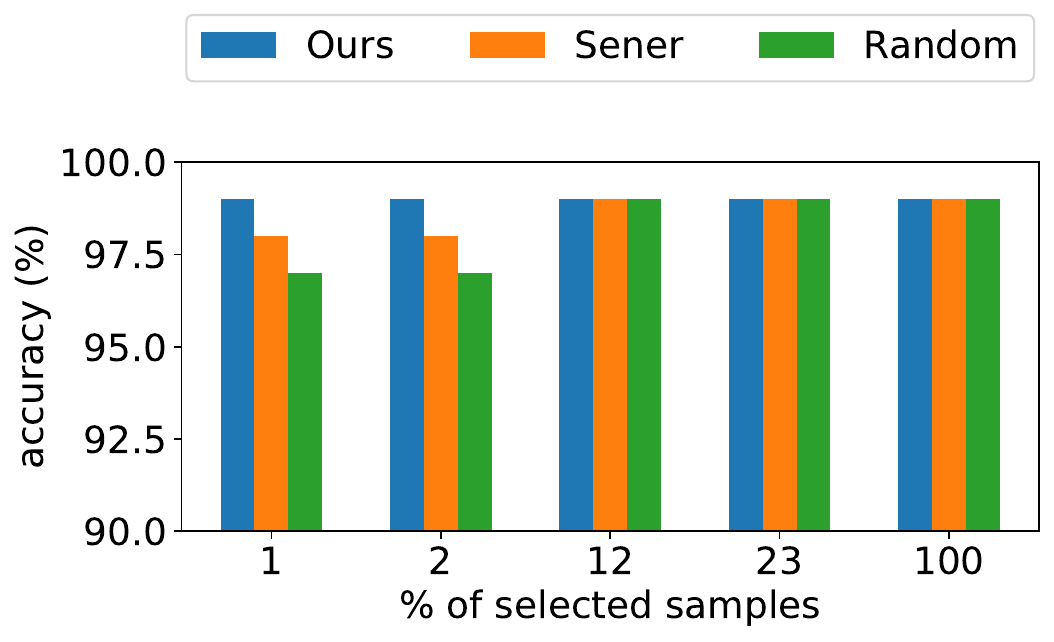}}}
\vspace{-.5em}
\centerline{(a) accuracy}
\end{minipage}\hfill
\begin{minipage}{.495\linewidth}
\centerline{\mbox{\includegraphics[width=1.0\linewidth,height=\figheight]
{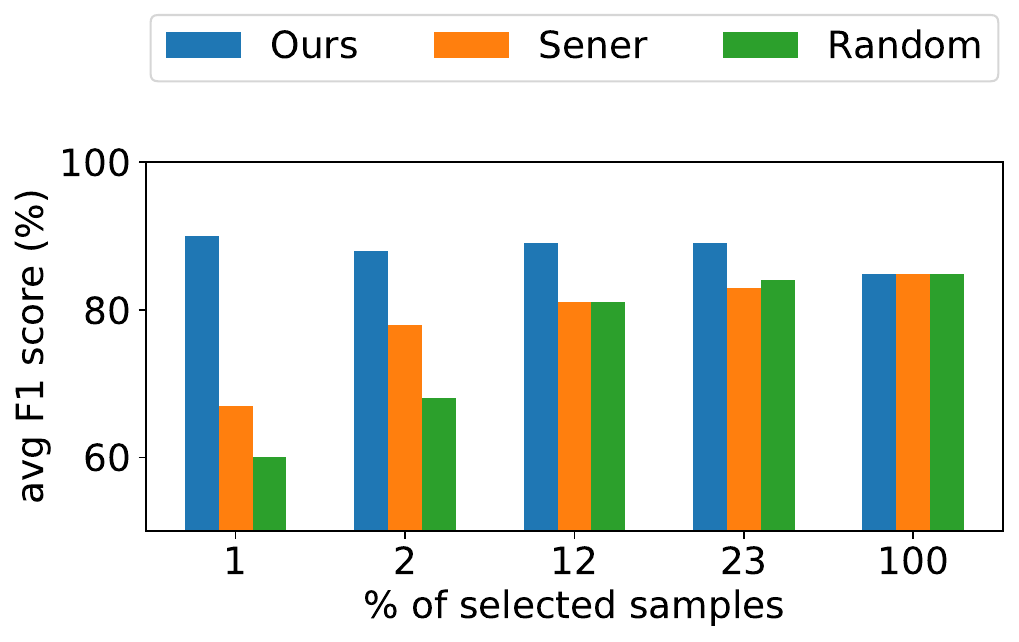}}}
\vspace{-.5em}
\centerline{(b) average F1 score}
\end{minipage}
\vspace{-1em}
\caption{\small {Quality of trained target model in Case~3. }}\label{fig:target model: case 3: n=5}
\end{figure}

\subsection{Prototype and Experimentation}\label{subsec:Prototype and Experimentation}

In addition to the public dataset, we also experiment on our own data collected through a prototype implementation. 

\subsubsection{Prototype Implementation}

\begin{figure}[!t]
   \centerline{\includegraphics[width=.9\linewidth]{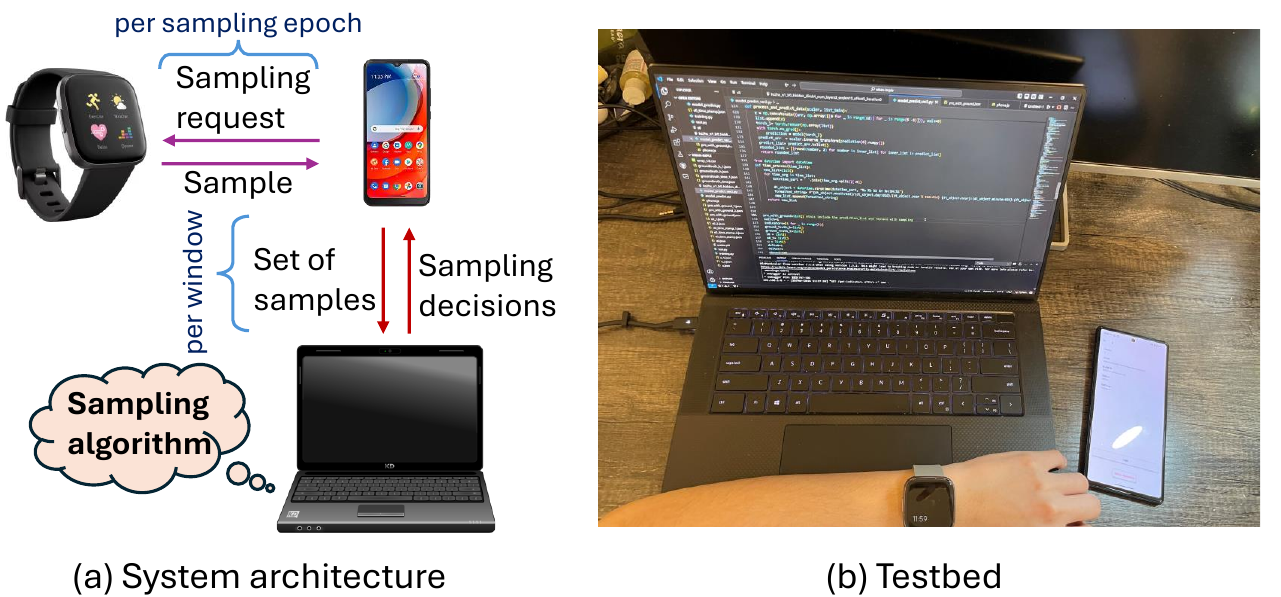}}
   \vspace{-1em}
    \caption{Prototype implementation. }
    \label{fig:prototype}
    \vspace{-.5em}
\end{figure}

We implement a prototype for a similar application scenario as in Section~\ref{subsec:Data-driven Simulation}, using a Google Pixel~6 as the coordinator and a FitBit Versa~2 as the sensing device. We use a workstation running PyTorch as the server for training and testing the forecasting model and the target model\footnote{The workstation also runs Flask for communication with the phone.}. As running the trained forecasting model requires  PyTorch runtime, which currently does not have a stable release for Andoid~\cite{PyTorchMobile}, we implement a workaround by running the forecasting model and the proposed sampling algorithm on the server, which results in the system architecture in Fig.~\ref{fig:prototype}. Our code is available on GitHub~\cite{WBANcode}. 

At the beginning of every prediction window of 5 minutes ($n=5$), the server runs Algorithm~\ref{Alg:predictive coreset} to select sampling epochs (each of 1 minute) in this window and sends the decision to the phone via WiFi. The phone then sends a sampling request to FitBit at the beginning of each selected sampling epoch via Bluetooth. When requested, FitBit reports the collected sample at the end of the epoch, which is then buffered at the phone and sent back to the server together with other collected samples at the end of the window. In real deployment, the forecasting model and the sampling algorithm should run on the phone, which only uploads the collected data for labeling and training sporadically (e.g., once per day). Nevertheless, our prototype will yield the same training results as the real deployment and we leave a deployable implementation to future work.  \looseness=-1

\subsubsection{Experiment Setting}

We use a setting similar to Section~\ref{subsubsec:Evaluation Setting - Simulation}. 
Specifically, we collect three types of data using FitBit API: heart rate every 5 seconds, step count every minute, and intensity level (0--3) every minute, where heart rate and step count are treated as input features of the target model and intensity level as the output. 
We collect a total of {14708} data points, each accounting for one minute, over the course of {17} days. We use 
{9555} of the data points to train the forecasting model (with {$80\%$} for training and {$20\%$} for validation) and the rest to train and test the target model (with $90\%$ for training and $10\%$ for testing). 
We use the same type of forecasting model as in Section~\ref{subsubsec:Evaluation Setting - Simulation}, which is trained to predict both heart rate (at 5-second granularity) and step count (at 1-minute granularity) for the next 5 minutes based on their values in the previous 5 minutes (as in Case~1). 
The target model is also of the same type as that in Section~\ref{subsubsec:Evaluation Setting - Simulation}, except that it uses the 12 heart rate measurements and the step count over one minute as input to estimate the intensity level in the same minute.  
Based on these, we evaluate the same set of sampling algorithms as in Section~\ref{subsubsec:Evaluation Setting - Simulation}. \looseness=-1

\subsubsection{Experiment Results}

\begin{figure}[t!]
\vspace{0em}
\begin{minipage}{.495\linewidth}
\centerline{\mbox{\includegraphics[width=1.0\linewidth,height=\figheight]
{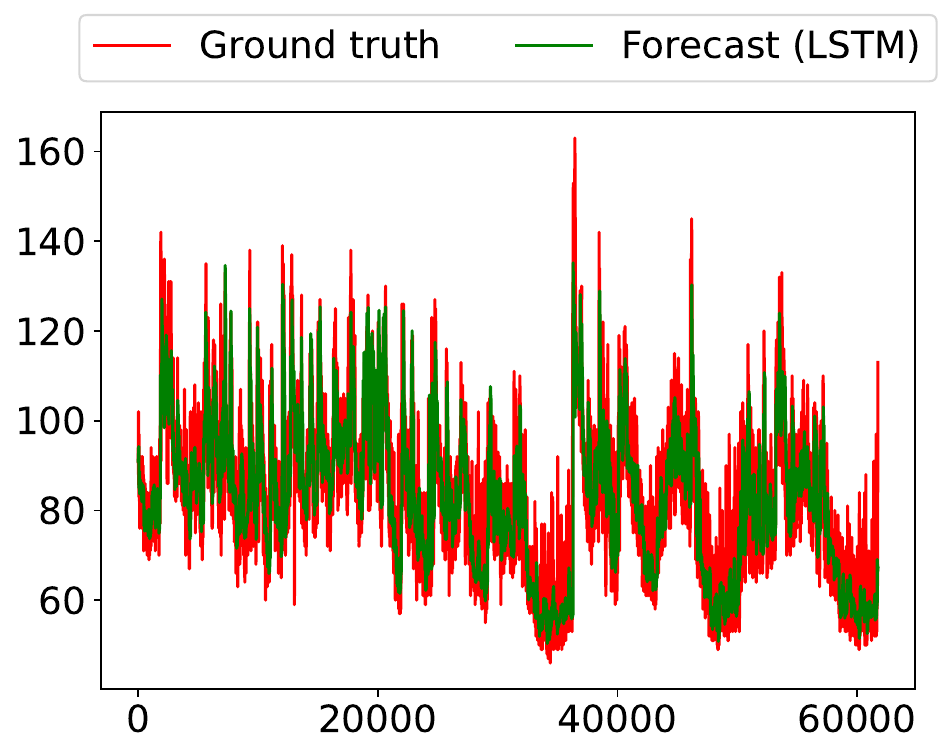}}}
\vspace{-.5em}
\centerline{(a) heart rate}
\end{minipage}\hfill
\begin{minipage}{.495\linewidth}
    \centerline{\mbox{\includegraphics[width=1.0\linewidth,height=\figheight]
{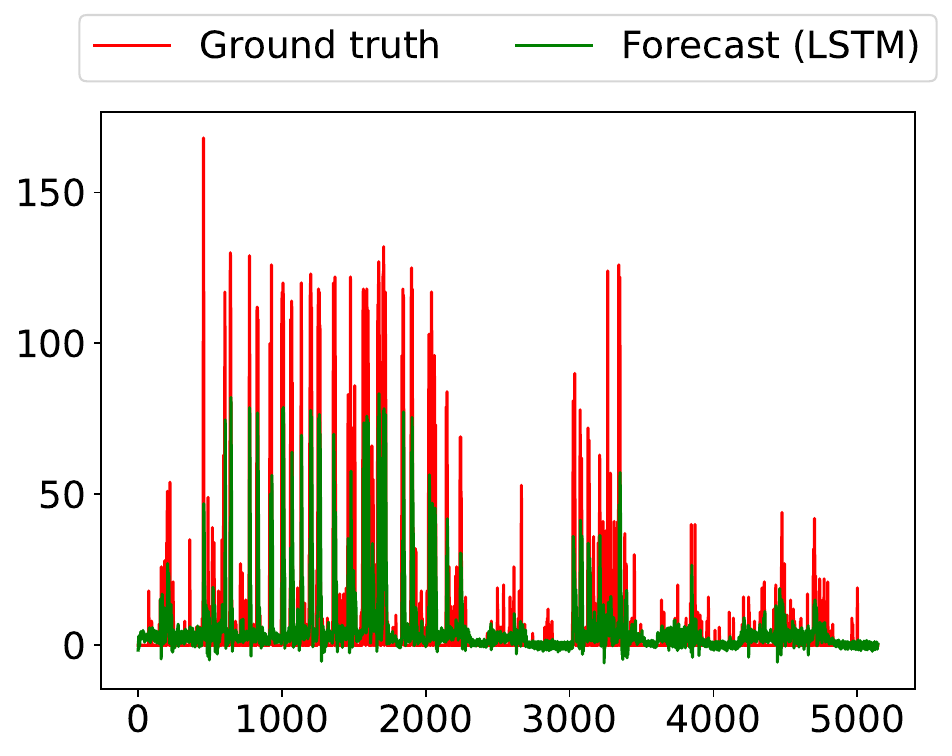}}}
\vspace{-.5em}
\centerline{(b) step count}
\end{minipage}
\vspace{-1em}
\caption{\small Results of time series forecasting on our data ($n=5$). }\label{fig:forecasting: case 4: n=5}
\vspace{-.5em}
\end{figure}

\begin{figure}[t!]
\vspace{-.0em}
\begin{minipage}{.495\linewidth}
\centerline{\mbox{\includegraphics[width=1.0\linewidth,height=\figheight]
{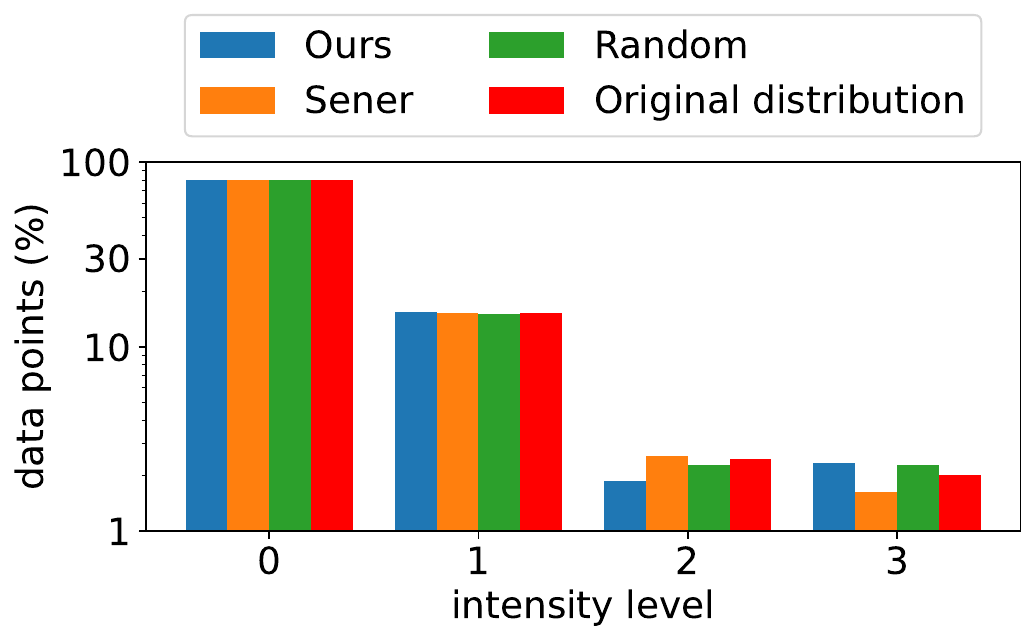}}}
\vspace{-.5em}
\centerline{(a) $4\%$ sampling}
\end{minipage}\hfill
\begin{minipage}{.495\linewidth}
\centerline{\mbox{\includegraphics[width=1.0\linewidth,height=\figheight]
{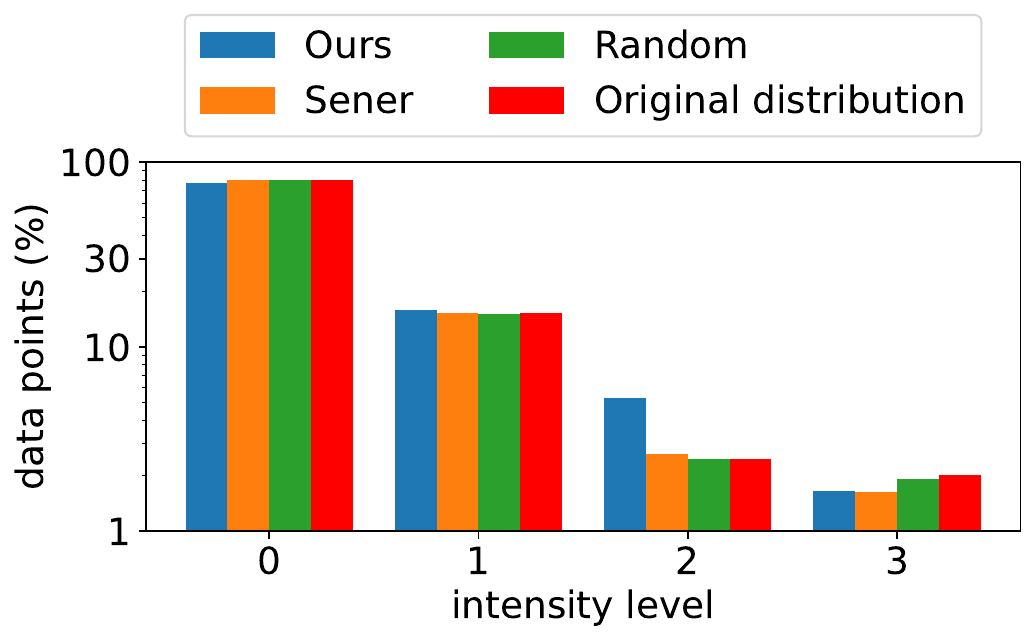}}}
\vspace{-.5em}
\centerline{(b) $8\%$ sampling}
\end{minipage}
\vspace{-1em}
\caption{\small { Distribution of selected samples on our data ($n=5$). } }\label{fig:sample distribution: case 4: n=5}
\vspace{-.25em}
\end{figure}

\begin{figure}[t!]
\vspace{-.0em}
\begin{minipage}{.495\linewidth}
\centerline{\mbox{\includegraphics[width=1.0\linewidth,height=\figheight]
{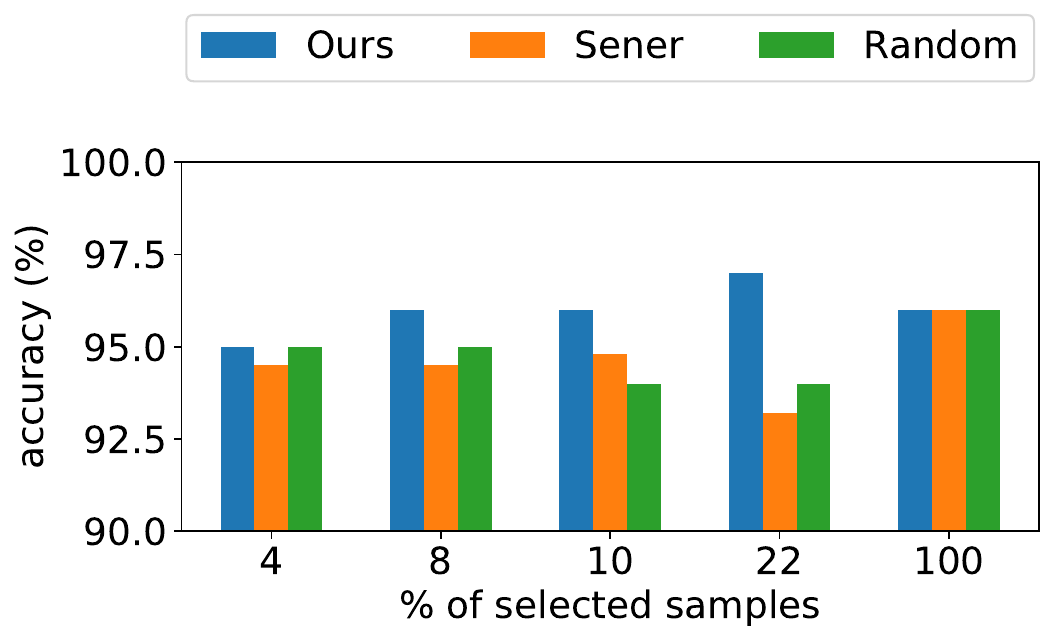}}}
\vspace{-.5em}
\centerline{(a) accuracy}
\end{minipage}\hfill
\begin{minipage}{.495\linewidth}
\centerline{\mbox{\includegraphics[width=1.0\linewidth,height=\figheight]
{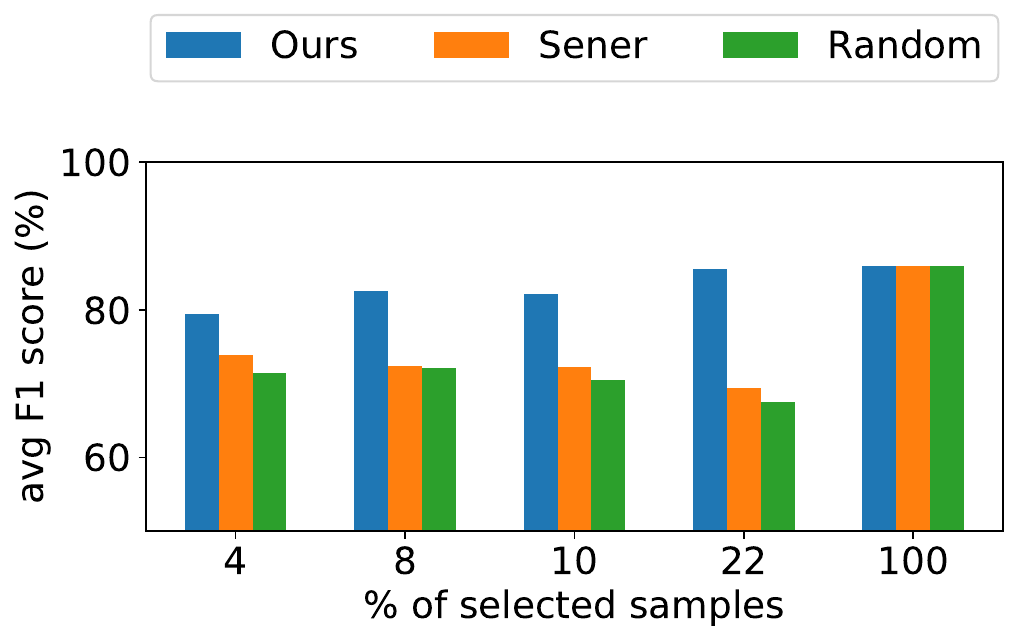}}}
\vspace{-.5em}
\centerline{(b) average F1 score}
\end{minipage}
\vspace{-1em}
\caption{\small {Quality of trained target model on our data ($n=5$). }}\label{fig:target model: case 4: n=5}
\end{figure}

\textbf{Results of forecasting:}
Fig.~\ref{fig:forecasting: case 4: n=5} shows the outputs of the trained forecasting model based on the true measurements. Note that the heart rate measurements are 5 seconds apart, and thus there are 12 times more heart rate measurements than step counts for the same measurement period. The results are consistent with those obtained from the public data, i.e., we can meaningfully predict health monitoring data in a sufficiently near future. The prediction error on our data is higher than that on the public data: the NRMSE is {$0.1052$} for heart rate and {$2.7212$} for step count. Nevertheless, we will show below that the noisy predictions still allow our algorithm to achieve notable performance improvement despite the increased error. 

\textbf{Results of sampling:} Fig.~\ref{fig:sample distribution: case 4: n=5} shows the distribution of the selected samples across different intensity levels in our dataset. Before sampling, our dataset is also skewed with a similar distribution as the public data. This skewness is inherited by random sampling and Sener's method \cite{Sener18ICLR}, but our algorithm can reduce the skewness when the sampling ratio is sufficiently high.

\textbf{Quality of trained model:} Fig.~\ref{fig:target model: case 4: n=5} shows the performance of the trained target model in terms of the classification accuracy and the average F1 score. Compared with Fig.~\ref{fig:target model: case 1: n=5}, Fig.~\ref{fig:target model: case 4: n=5} confirms that our algorithm can outperform the benchmarks and substantially reduce the data collection and labeling cost (by $92\%$) without notably degrading the quality of the trained model. 
Meanwhile, the sampling ratio required to maintain the same model quality as training on the full data has increased from around $2\%$ in Fig.~\ref{fig:target model: case 1: n=5} to around $8\%$ in Fig.~\ref{fig:target model: case 4: n=5}. 
%
This is because our dataset is smaller, only containing 4,653 minutes of data for training the target model as opposed to nearly 20,000 minutes as in the public dataset, and thus the sampling ratio for our dataset has to increase accordingly to yield the same amount of training data. 
Comparing Fig.~\ref{fig:sample distribution: case 4: n=5}--\ref{fig:target model: case 4: n=5} under $4\%$ sampling also shows that our improvement is not just from correcting data skewness, as our algorithm can outperform random sampling even when its sampled data are no less skewed. 

\begin{figure}[t!]
\vspace{0em}
\begin{minipage}{.495\linewidth}
\centerline{\mbox{\includegraphics[width=1.0\linewidth,height=\figheight]
{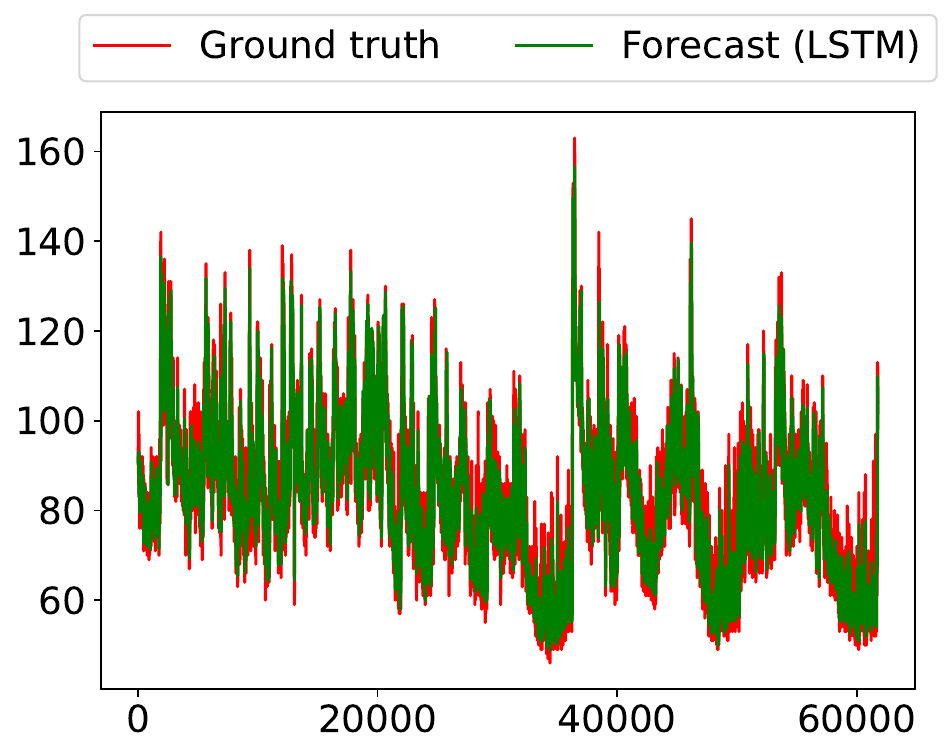}}}
\vspace{-.5em}
\centerline{(a) heart rate}
\end{minipage}\hfill
\begin{minipage}{.495\linewidth}
    \centerline{\mbox{\includegraphics[width=1.0\linewidth,height=\figheight]
{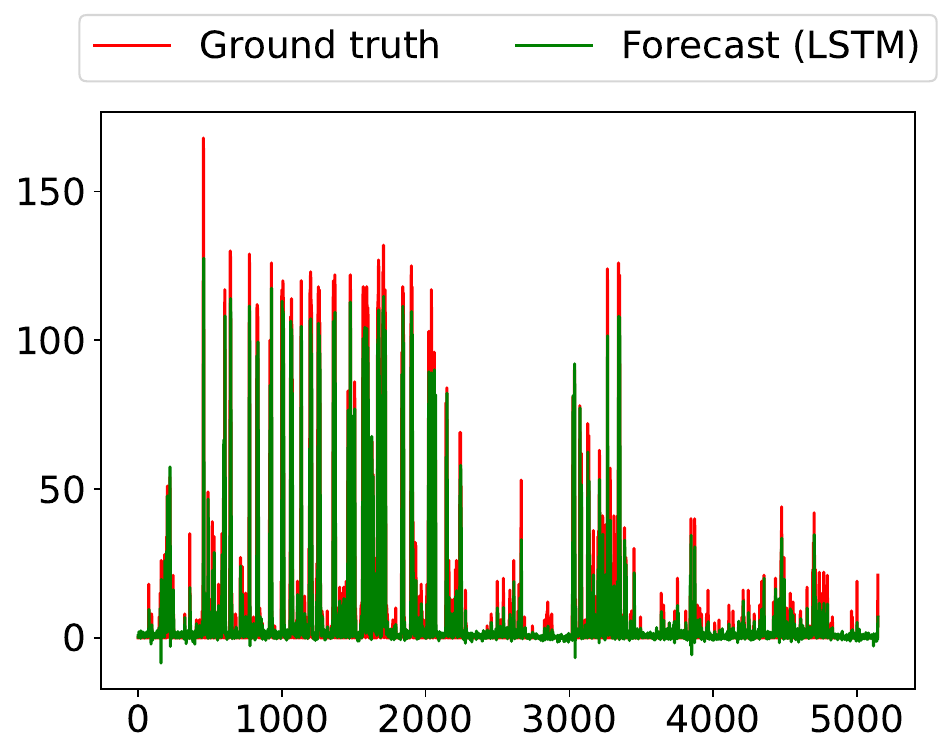}}}
\vspace{-.5em}
\centerline{(b) step count}
\end{minipage}
\vspace{-1em}
\caption{\small Results of time series forecasting on our data ($n=1$).}\label{fig:forecasting: case 4: n=1}
\vspace{-.5em}
\end{figure}

\begin{figure}[t!]
\vspace{0em}
\begin{minipage}{.495\linewidth}
\centerline{\mbox{\includegraphics[width=1.0\linewidth,height=\figheight]
{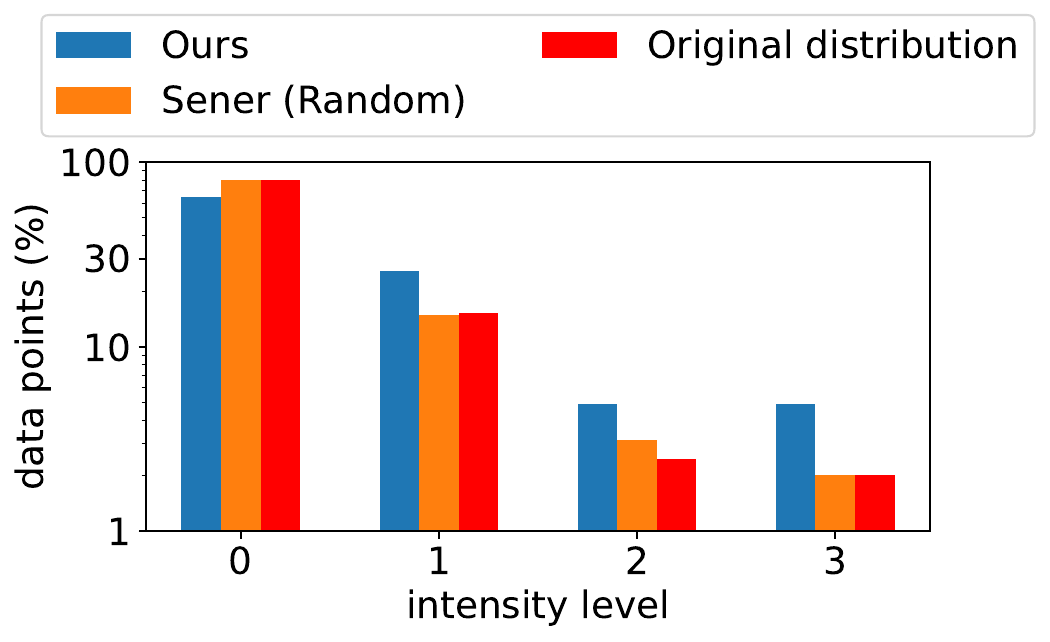}}}
\vspace{-.5em}
\centerline{(a) $3\%$ sampling}
\end{minipage}\hfill
\begin{minipage}{.495\linewidth}
\centerline{\mbox{\includegraphics[width=1.0\linewidth,height=\figheight]
{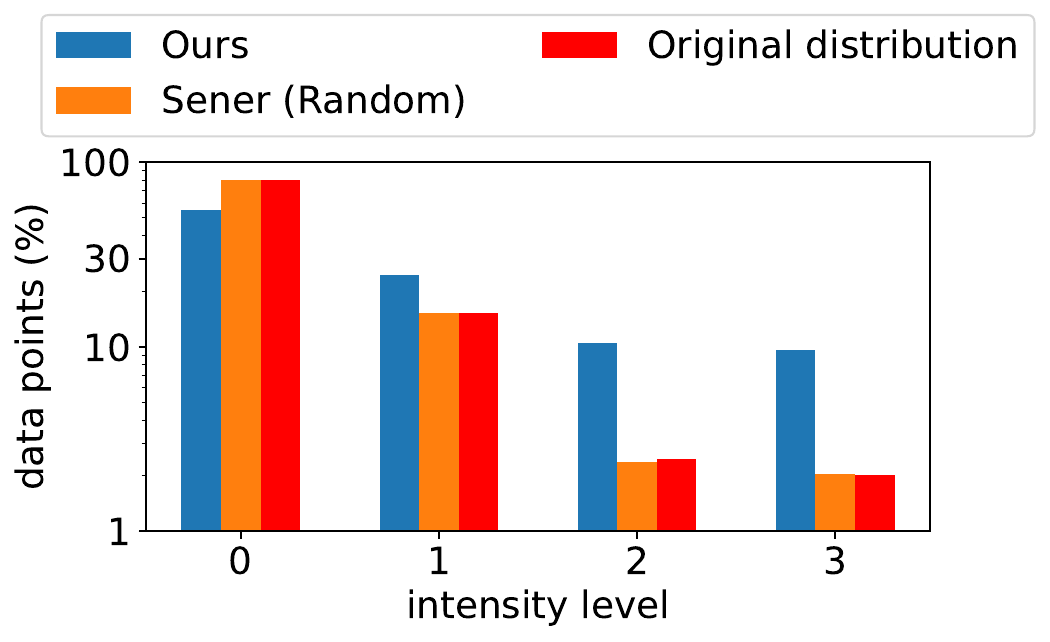}}}
\vspace{-.5em}
\centerline{(b) $12\%$ sampling}
\end{minipage}
\vspace{-1em}
\caption{\small {Distribution of selected samples on our data ($n=1$). } }\label{fig:sample distribution: case 4: n=1}
\end{figure}

\begin{figure}[t!]
\vspace{-.0em}
\begin{minipage}{.495\linewidth}
\centerline{\mbox{\includegraphics[width=1.0\linewidth,height=\figheight]
{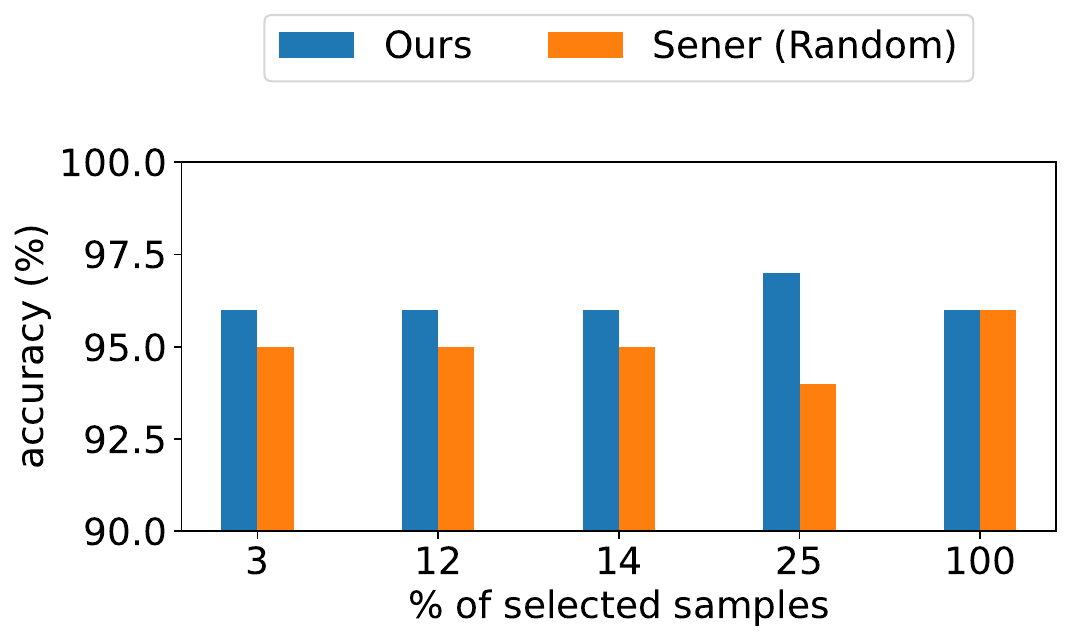}}}
\vspace{-.5em}
\centerline{(a) accuracy}
\end{minipage}\hfill
\begin{minipage}{.495\linewidth}
\centerline{\mbox{\includegraphics[width=1.0\linewidth,height=\figheight]
{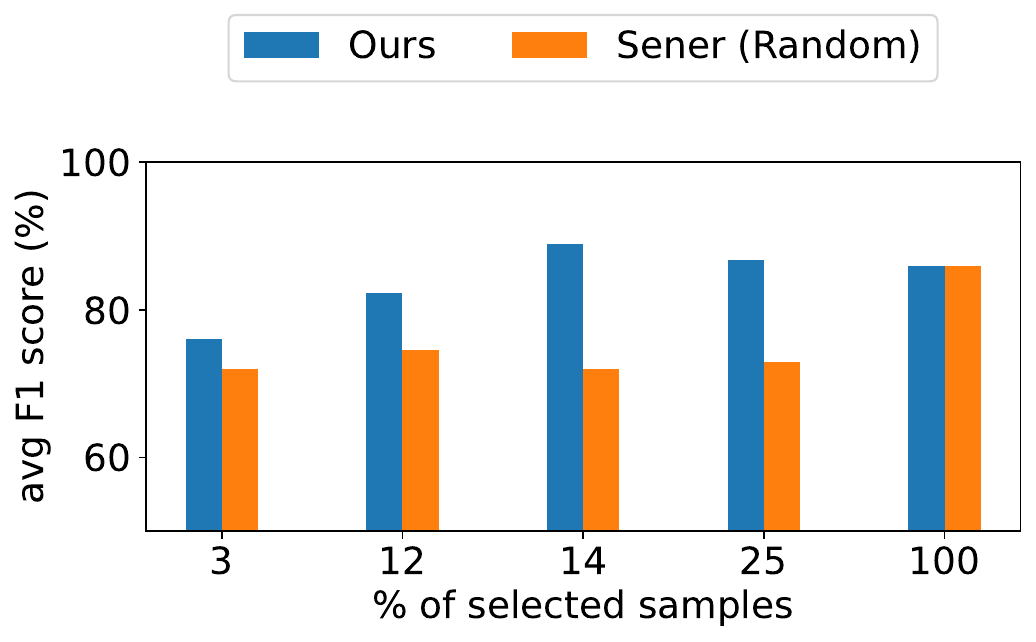}}}
\vspace{-.5em}
\centerline{(b) average F1 score}
\end{minipage}
\vspace{-1em}
\caption{\small {Quality of trained target model on our data ($n=1$).} }\label{fig:target model: case 4: n=1}
\end{figure}

\if\thisismainpaper0
\begin{figure}[t!]
\vspace{0em}
\begin{minipage}{.495\linewidth}
\centerline{\mbox{\includegraphics[width=1.0\linewidth,height=\figheight]
{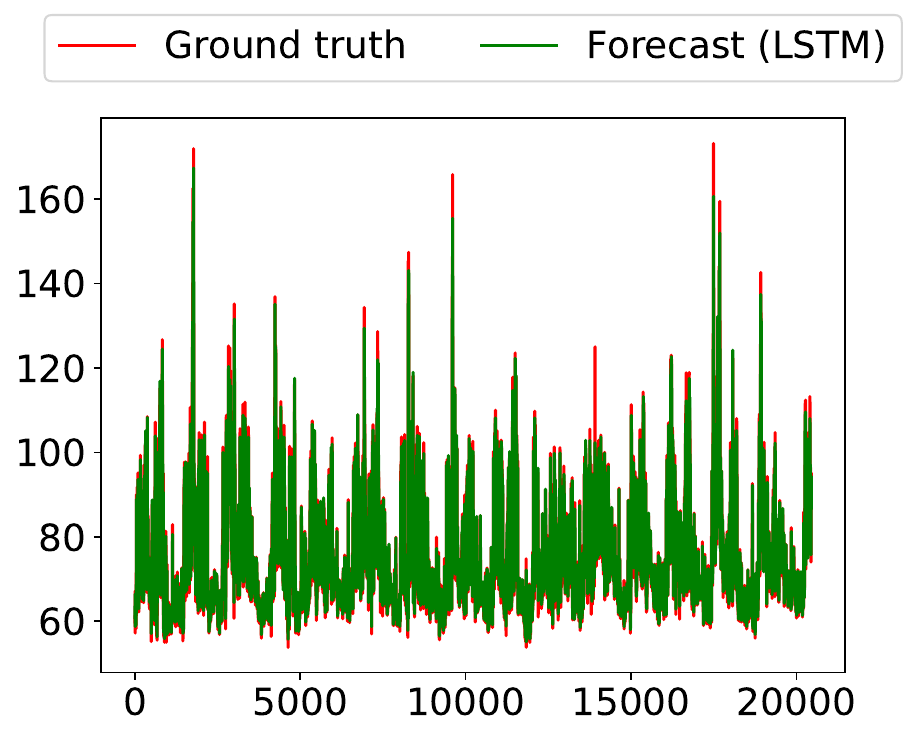}}}
\vspace{-.5em}
\centerline{(a) heart rate}
\end{minipage}\hfill
\begin{minipage}{.495\linewidth}
    \centerline{\mbox{\includegraphics[width=1.0\linewidth,height=\figheight]
{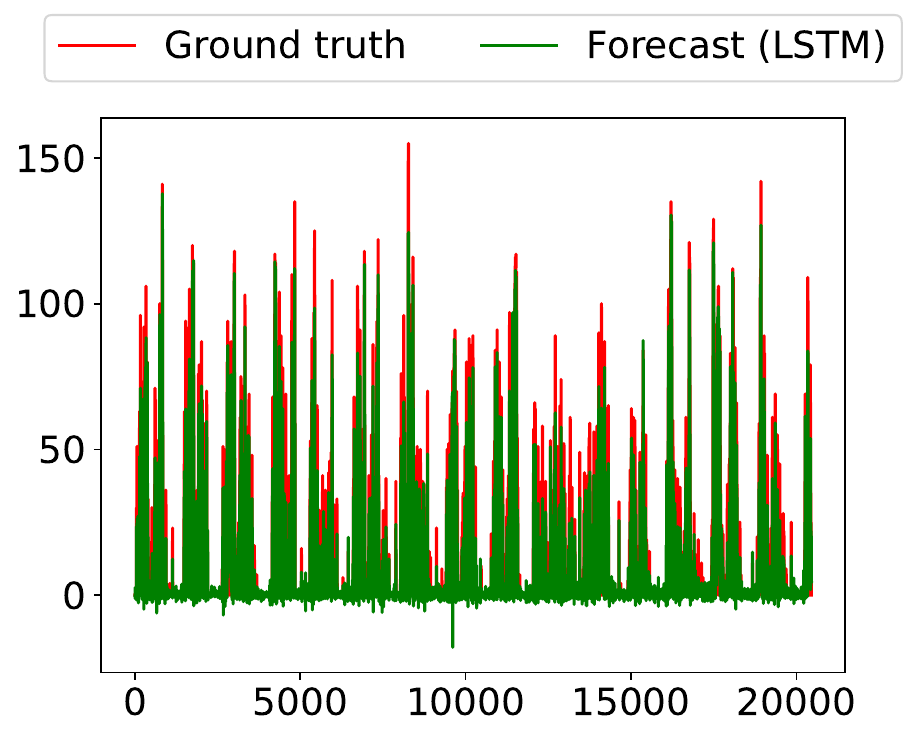}}}
\vspace{-.5em}
\centerline{(b) step count}
\end{minipage}
\vspace{-1em}
\caption{\small Results of time series forecasting on public data ($n=1$).}\label{fig:forecasting: case 1: n=1}
\end{figure}

\begin{figure}[t!]
\vspace{0em}
\begin{minipage}{.495\linewidth}
\centerline{\mbox{\includegraphics[width=1.0\linewidth,height=\figheight]
{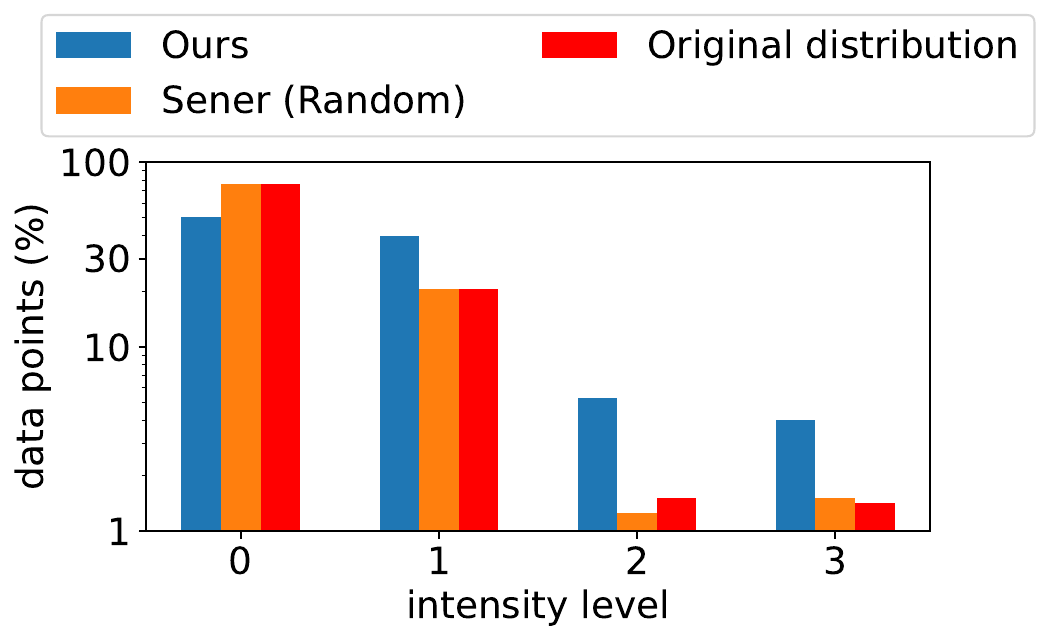}}}
\vspace{-.5em}
\centerline{(a) $6\%$ sampling}
\end{minipage}\hfill
\begin{minipage}{.495\linewidth}
\centerline{\mbox{\includegraphics[width=1.0\linewidth,height=\figheight]
{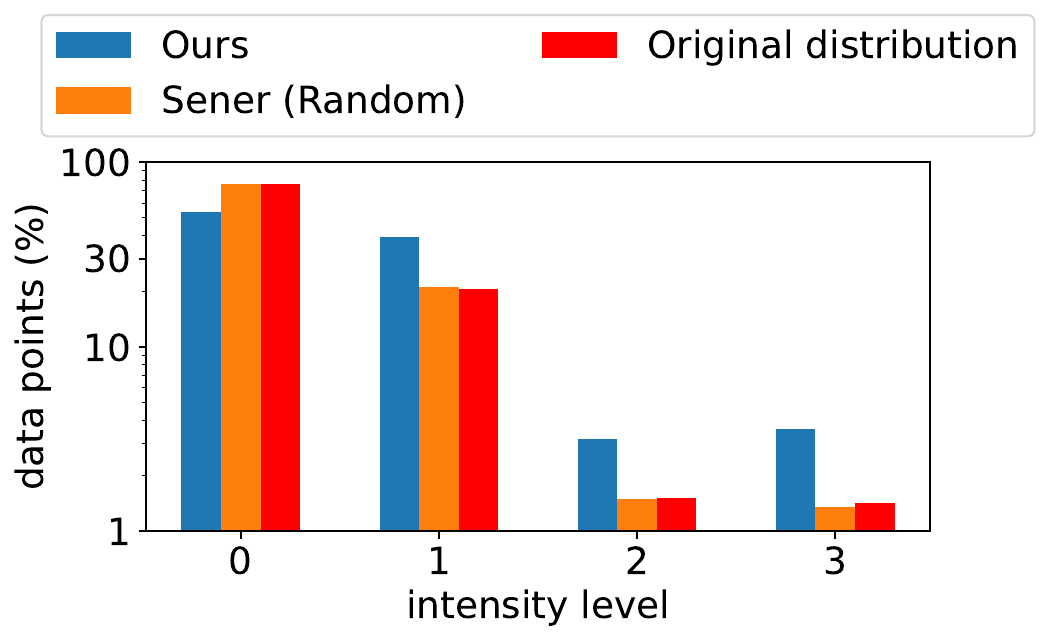}}}
\vspace{-.5em}
\centerline{(b) $12\%$ sampling}
\end{minipage}
\vspace{-1em}
\caption{\small {Distribution of selected samples on public data ($n=1$). 
} }\label{fig:sample distribution: case 1: n=1}
\end{figure}

\begin{figure}[t!]
\vspace{-.0em}
\begin{minipage}{.495\linewidth}
\centerline{\mbox{\includegraphics[width=1.0\linewidth,height=\figheight]
{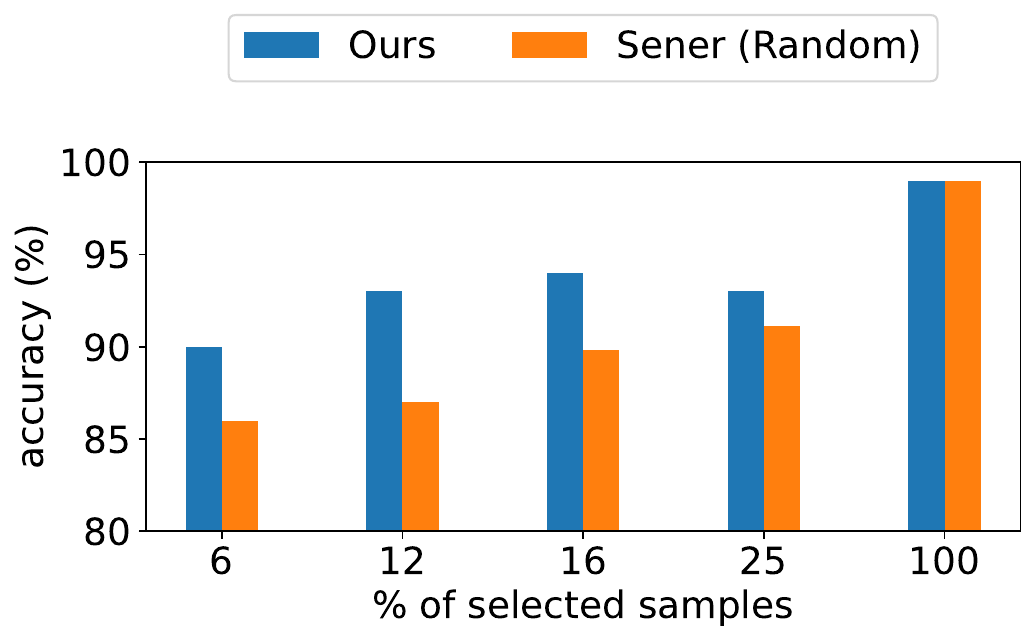}}}
\vspace{-.5em}
\centerline{(a) accuracy}
\end{minipage}\hfill
\begin{minipage}{.495\linewidth}
\centerline{\mbox{\includegraphics[width=1.0\linewidth,height=\figheight]
{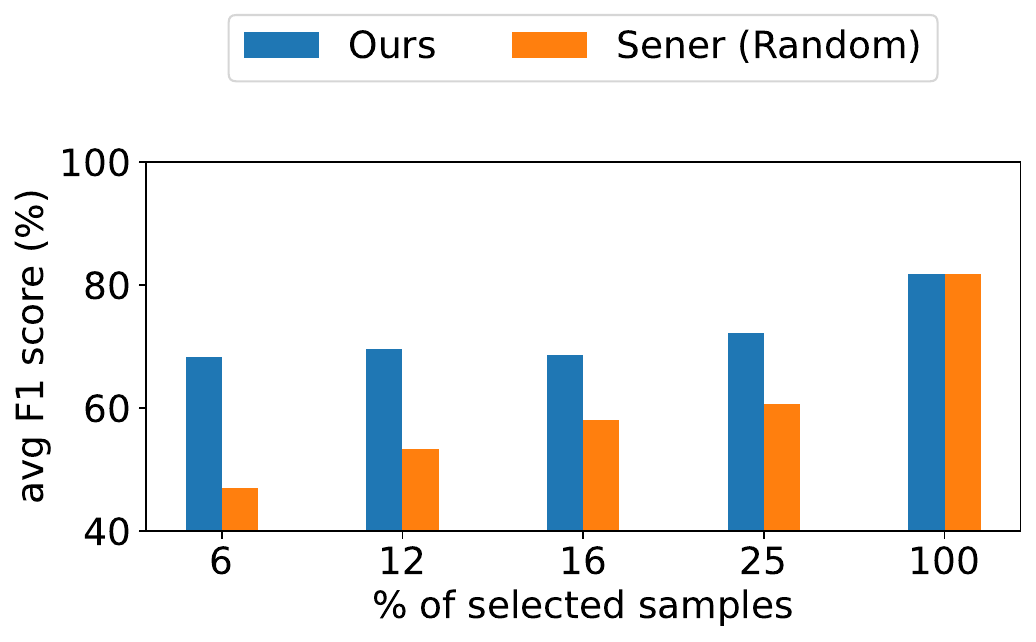}}}
\vspace{-.5em}
\centerline{(b) average F1 score}
\end{minipage}
\vspace{-1em}
\caption{\small {Quality of trained target model on public data ($n=1$). 
}}\label{fig:target model: case 1: n=1}
\end{figure}

\fi

\textbf{Impact of prediction window size:}
While the above results are obtained under a window size of 5 minutes ($n=5$), we have verified that the observations remain largely the same under other window sizes, as long as reasonable predictions can be made. A case of special interest is when the window size is 1 minute ($n=1$), in which case Sener's method~\cite{Sener18ICLR} degenerates into random sampling as each window contains only one sampling epoch. In this case, we can predict heart rate and step count with higher accuracy as shown in Fig.~\ref{fig:forecasting: case 4: n=1}: the NRMSE is $0.0706$ for heart rate and $1.8309$ for step count. 
The improved prediction allows our sampling algorithm to more accurately select the representative samples, thus more effectively correcting the data skewness as shown in Fig.~\ref{fig:sample distribution: case 4: n=1}. The more balanced training set then leads to a better target model as shown in Fig.~\ref{fig:target model: case 4: n=1}. Compared to the case of $n=5$ in Fig.~\ref{fig:forecasting: case 4: n=5}--\ref{fig:target model: case 4: n=5}, setting $n=1$ leads to better prediction and skewness correction, but less flexibility in making sampling decisions as the decisions are made one sample at a time, and the final target model qualities turn out to be similar for our dataset. 
\if\thisismainpaper1
However, the outcome of this comparison is data-dependent as shown in \cite{Chiu24:report}. Thus, the window size $n$ should be tuned as a hyperparameter. 
\else
For comparison, we also evaluate the case of $n=1$ on the public data  from Section~\ref{subsec:Data-driven Simulation} (Case~1) and show the corresponding results in Fig.~\ref{fig:forecasting: case 1: n=1}--\ref{fig:target model: case 1: n=1}. Compared to the results under $n=5$ in Fig.~\ref{fig:target model: case 1: n=5}, we see that reducing the window size to $n=1$ notably reduces the quality of the target model, meaning that the flexibility of sampling decisions outweighs the accuracy of prediction for this dataset\footnote{Compared to the case of $n=5$, setting $n=1$ for the public data reduces the NRMSE from $0.0855$ to $0.0524$ for heart rate and from $1.9339$ to $1.4295$ for step count.}. Thus, the optimal window size will be data-dependent and should be tuned as a hyperparameter. 
\fi

\emph{Remark:} We observe in both the public dataset and our own data that the raw data exhibit significant skewness in terms of the data distribution across labels. This is a common phenomenon in health monitoring, as some health states are expected to occur more often than others. In this regard, our coreset-based solution has demonstrated the ability to reduce the data skewness during data collection \emph{before} going through the slow and expensive labeling process. 
\rev{Meanwhile, our solution requires the sensor measurements to be predictable with reasonable accuracy in a sufficiently small time window, which is typical for continuous health monitoring, while our solution is applicable for active learning on any time series with sufficient predictability.}

\section{Conclusion}\label{sec:Conclusion}

Motivated by the high data collection cost in WBAN, we studied a novel active learning problem of selectively collecting data from remote data sources for  training a classification model without real-time access to labels. Based on a novel approximation error bound, we converted the active learning problem into a predictive coreset construction problem, for which we developed an algorithm with guaranteed performance. Our evaluation based on both public health monitoring data and our own experiment demonstrated the superior capability of our algorithm in reducing the data curation cost without sacrificing the quality of the trained model. 
With the rapid development of artificial intelligence (AI) applications, we envision the proposed solution to play an active role in providing personalized AI models in resource-constrained environments. 
\looseness=-1



\bibliographystyle{ACM-Reference-Format}
\bibliography{references.bib}

\if\thisismainpaper1

\else

\appendix

\section{Appendix}\label{appendix:A}

\subsection{Implementation of Algorithm~\ref{Alg:predictive coreset}}\label{appendix:Implementation of Algorithm}

Here we will describe a method to solve line~\ref{sampling:4} optimally in $O(n^3 2^n)$ time. As shown in \eqref{eq:optimization for sampling4}, line~\ref{sampling:4} is an ILP with two types of decision variables: the sampling variables $(\alpha_j)_{j\in [n]}$ and the coverage variables $(\beta_{ij})_{i: \hat{\xbf}_i\in Q, j\in [n]}$ ($[n]:=\{1,\ldots,n\}$). We brute-force search all the $2^n$ possible solutions to the sampling variables in increasing order of $\sum_{j=1}^n \alpha_j$ and test the feasibility of each solution with respect to \eqref{eq:optimization for sampling4} until finding a feasible solution, which is guaranteed to be optimal for \eqref{eq:optimization for sampling4}. It thus suffices to show that for a given solution to $(\alpha_j)_{j\in [n]}$, the existence of a solution to $(\beta_{ij})_{i: \hat{\xbf}_i\in Q, j\in [n]}$ that satisfies the constraints of \eqref{eq:optimization for sampling4} can be tested in $O(n^3)$ time. 

\begin{figure}[thb]
\vspace{-.0em}
\centerline{\mbox{\includegraphics[width=.5\linewidth]{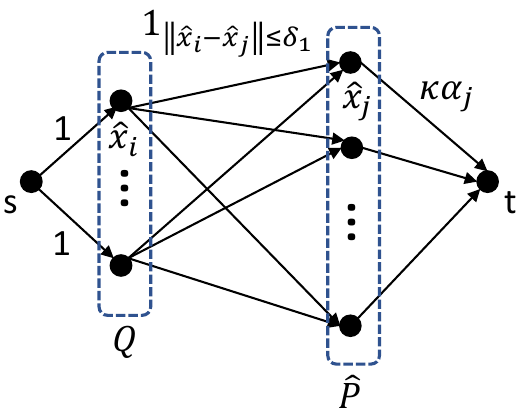}}}
\vspace{-1em}
\caption{\small Auxiliary graph $\mathcal{G}$ for reduction to the max-flow problem. }
\label{fig:max_flow}
\end{figure}

Our idea is to convert the feasibility test to a max-flow problem in an auxiliary graph. As illustrated in Fig.~\ref{fig:max_flow}, we construct an auxiliary graph $\mathcal{G}$ based on $Q, \hat{P}$, and $(\alpha_j)_{j\in [n]}$, which consists of a complete bipartite graph between nodes representing elements of $Q$ and nodes representing elements of $\hat{P}$, a source node $s$ connected to the nodes representing $Q$, and a destination node $t$ connected to the nodes representing $\hat{P}$. All the links are directed with integral capacities as marked on the links. We can test the feasibility of $(\alpha_j)_{j\in [n]}$ by computing the maximum flow from $s$ to $t$ and comparing it with $|Q|$, as stated below. 

\begin{lemma}\label{lem:max flow test}
A given solution to $(\alpha_j)_{j\in [n]}$ is feasible for \eqref{eq:optimization for sampling4} if and only if the maximum $s\to t$ flow in the corresponding auxiliary graph $\mathcal{G}$ as constructed in Fig.~\ref{fig:max_flow} equals $|Q|$. 
\end{lemma}
\begin{proof}
We will show that the existence of a solution to $(\beta_{ij})_{i: \hat{\xbf}_i\in Q, j\in [n]}$ that together with the given $(\alpha_j)_{j\in [n]}$ satisfies all the constraints of \eqref{eq:optimization for sampling4} is equivalent to the existence of a maximum $s\to t$ flow that equals $|Q|$.

If there exists  $(\beta_{ij})_{i: \hat{\xbf}_i\in Q, j\in [n]}$ that together with $(\alpha_j)_{j\in [n]}$ satisfies all the constraints of \eqref{eq:optimization for sampling4}, then we can construct a feasible $s\to t$ flow in $\mathcal{G}$ by sending a unit flow on the link from $s$ to each $\hat{\xbf}_i\in Q$, a flow of $\beta_{ij}$ on the link from each $\hat{\xbf}_i\in Q$ to each $\hat{\xbf}_j\in \hat{P}$, and a flow of $\sum_{i: \hat{\xbf}_i\in Q}\beta_{ij}$ on the link from each $\hat{\xbf}_j\in \hat{P}$ to $t$. Since the minimum cut between $s$ and $t$ is no larger than $|Q|$, this flow is the maximum flow from $s$ to $t$, whose rate equals $|Q|$.

If there exists a maximum $s\to t$ flow that equals $|Q|$, then by the  \emph{Integral Flow Theorem} \cite{Korte:NetworkFlows}, there must exist a maximum flow from $s$ to $t$ that has an integral flow rate on each link, as all the link capacities in $\mathcal{G}$ are integers. Under this maximum integral flow, each link from $\hat{\xbf}_i\in Q$ to $\hat{\xbf}_j\in \hat{P}$ must carry either $0$ or $1$ unit of flow. If we set each $\beta_{ij}$ to the flow rate on the link from $\hat{\xbf}_i\in Q$ to $\hat{\xbf}_j\in \hat{P}$, then the flow conservation constraint and the link capacity constraint imply that $(\beta_{ij})_{i: \hat{\xbf}_i\in Q, j\in [n]}$  together with the given $(\alpha_j)_{j\in [n]}$ satisfies all the constraints of \eqref{eq:optimization for sampling4}.
\end{proof}

Therefore, the feasibility test is reduced to a computation of the maximum flow in the auxiliary graph, which can be implemented by the Ford-Fulkerson algorithm with time complexity $O(|E_{\mathcal{G}}| f_{\mathcal{G}})$ \cite{Korte:NetworkFlows}, where $|E_{\mathcal{G}}|$ is the number of links in $\mathcal{G}$ and $f_{\mathcal{G}}$ the maximum flow. Since $|E_{\mathcal{G}}|=O(|Q|\cdot |\hat{P}|)=O(n^2)$ and $f_{\mathcal{G}}\leq |Q| = O(n)$, the maximum $s\to t$ flow in the auxiliary graph can be computed in $O(n^3)$ time.

\subsection{Additional Evaluation Results}\label{appendix:Additional Evaluation}

We now investigate the influence of the parameter $\kappa$. Recall that  $\kappa$ represents the maximum number of data points represented by each selected sample in the coreset. Consequently, a smaller value of $\kappa$ necessitates that Algorithm~\ref{Alg:predictive coreset} selects a larger number of samples, which subsequently leads to an elevated F1 score, 
as shown in Fig.~\ref{fig:kappa} based on the public dataset in the scenario of $n=1$ and Case~1. We have repeated this experiment under all the scenarios, and all the results exhibit the same trend as Fig.~\ref{fig:kappa} (omitted).

We have also tested different methods to set the input parameters for Algorithm~\ref{Alg:predictive coreset}. Consider the case of $n=1$ for simplicity. Fig.~\ref{fig:kappa and delta} shows the comparison between two  methods of controlling the sampling ratio: one method varies $\delta_0$ with $\kappa$ set to infinity (`$\kappa=\infty$'), and the other method varies $\kappa$ with $\delta_0$ set to infinity (`$\delta_0=\infty$'). When $\kappa<\infty$ and $\delta_0 = \infty$, Algorithm~\ref{Alg:predictive coreset} is reduced to periodic sampling, where the period is decided by the value of $\kappa$. As shown in the Fig.~\ref{fig:kappa and delta}, `$\kappa=\infty$' outperforms `$\delta_0=\infty$'. This is intuitively because a finite coverage radius will enable Algorithm~\ref{Alg:predictive coreset} to select the samples that are sufficiently distinct from each other, which are potentially more informative for target model training. In general, we may need to set both $\kappa$ and $\delta_i$ ($i=0,1$) to finite values to achieve the optimal cost-quality tradeoff. In this sense, the performance of Algorithm~\ref{Alg:predictive coreset} shown in Section~\ref{sec:Performance Evaluation} is a lower bound of what our solution can achieve, and we leave further optimizations to future work. \looseness=-1

\begin{figure}[t!]
\vspace{-.5em}
\begin{minipage}{.495\linewidth}
\centerline{\mbox{\includegraphics[width=1.0\linewidth,height=\figheight]
{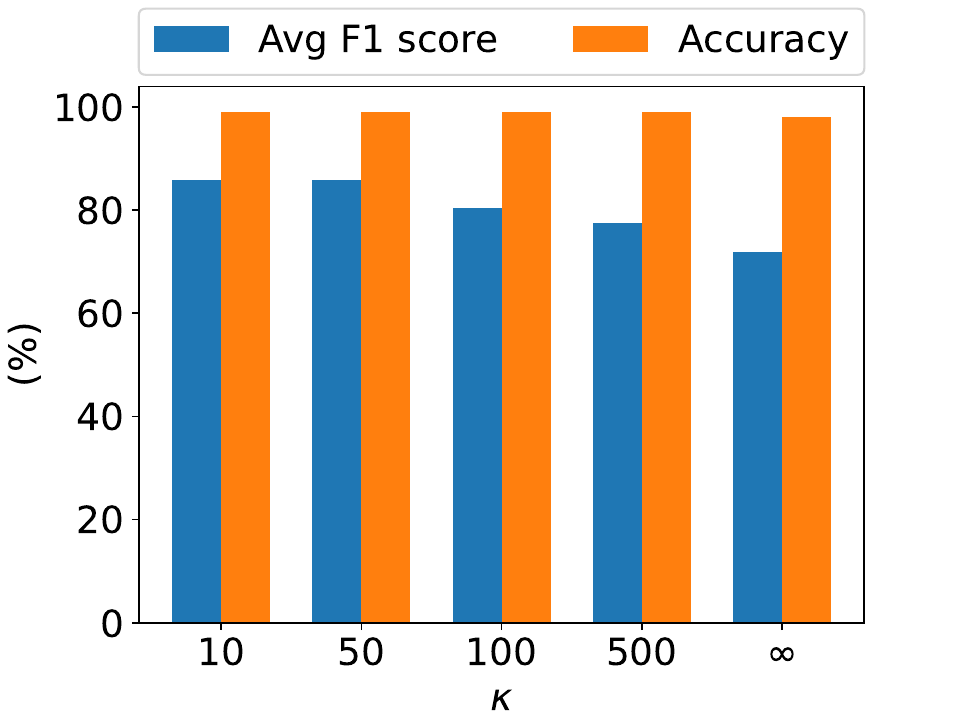}}}
\vspace{-.5em}
\centerline{(a) performance}
\end{minipage}\hfill
\begin{minipage}{.495\linewidth}
\centerline{\mbox{\includegraphics[width=1.0\linewidth,height=\figheight]
{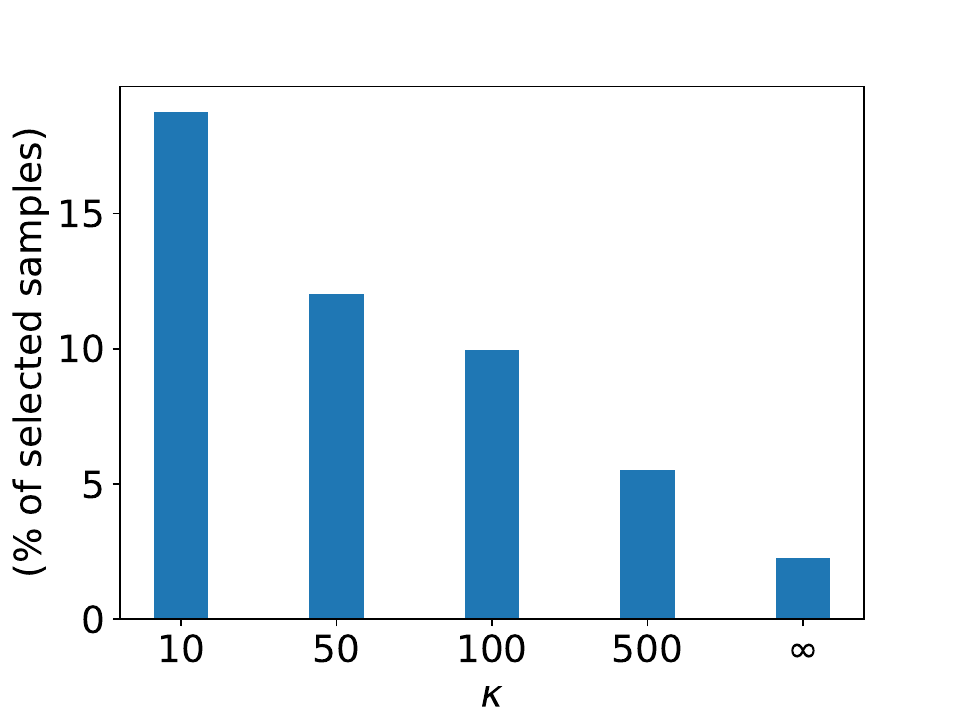}}}
\vspace{-.5em}
\centerline{(b) sample size}
\end{minipage}
\vspace{-1em}
\caption{\small Effect of varying $\kappa$ (public dataset, Case~1, $n=1$, $\delta_0 = 0.4$). }\label{fig:kappa}
\end{figure}

\begin{figure}[t!]
\vspace{-.5em}
\begin{minipage}{.495\linewidth}
\centerline{\mbox{\includegraphics[width=1.0\linewidth,height=\figheight]
{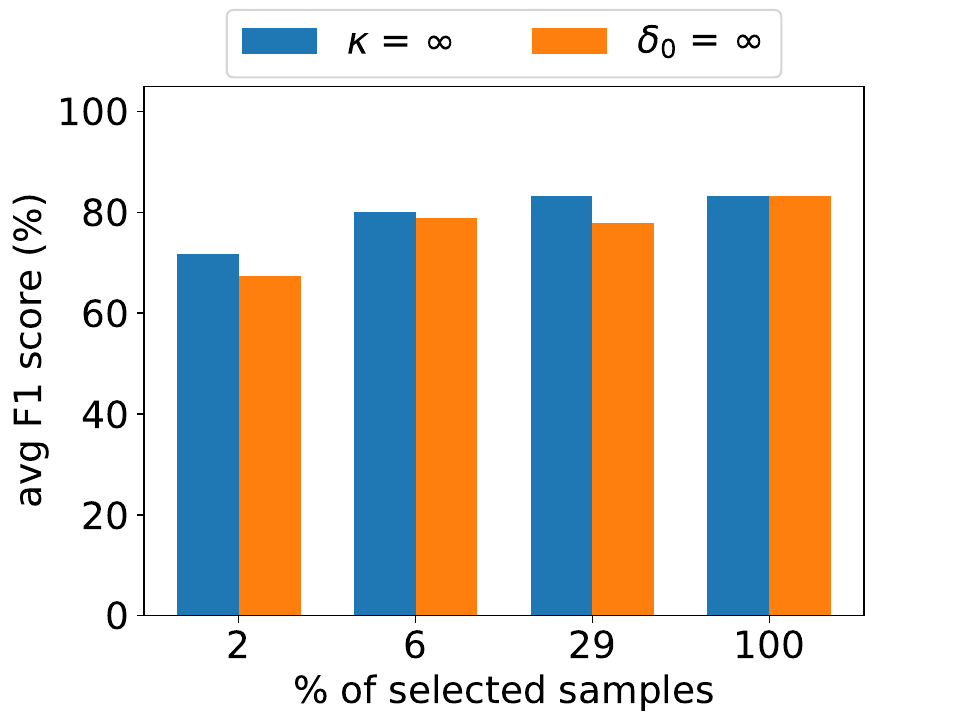}}}
\vspace{-.5em}
\centerline{(a) Case 1}
\end{minipage}\hfill
\begin{minipage}{.495\linewidth}
\centerline{\mbox{\includegraphics[width=1.0\linewidth,height=\figheight]
{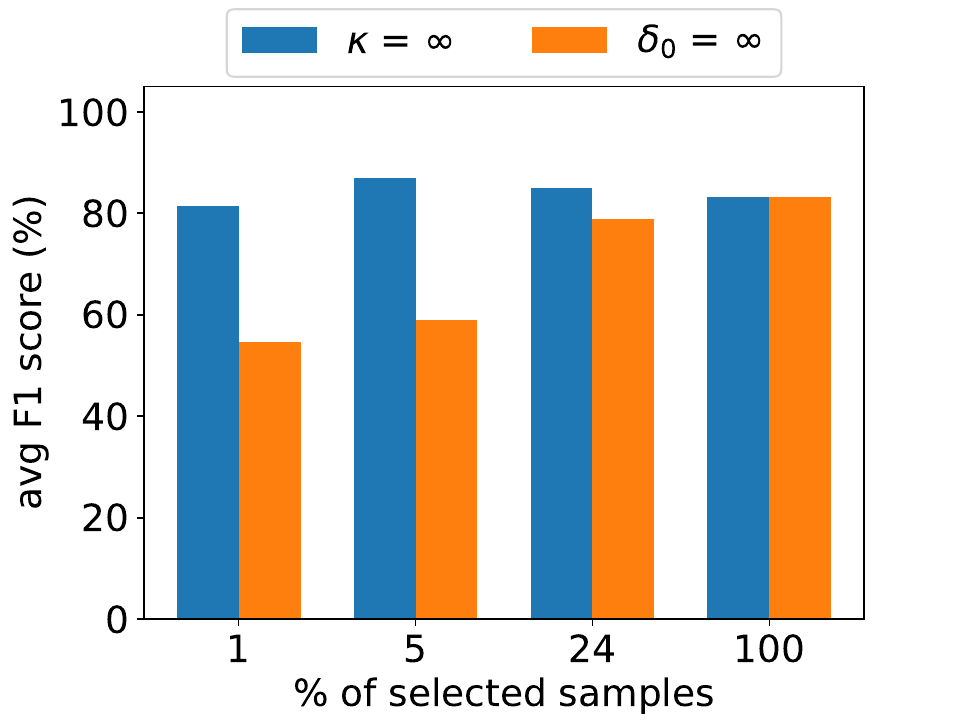}}}
\vspace{-.5em}
\centerline{(b) Case 2}
\end{minipage}
\vspace{-1em}
\caption{\small Tuning the parameters of Algorithm~\ref{Alg:predictive coreset} on public dataset. }\label{fig:kappa and delta}
\end{figure}

\subsection{Supporting Proofs}\label{appendix:Supporting Proofs}

\begin{proof}[Proof of Theorem~\ref{thm:approx error bound}]
Let \[\mbbE_{y_i}[l(\xbf,y_i;\wbf)] := \sum_{c=1}^C \eta_c(\xbf_i)l(\xbf,c;\wbf).\] For any $\xbf_i\in P$, 
\begin{align}
    \mbbE_{y_i}[l(\xbf_i,y_i;\wbf)] &=\sum_{c=1}^C \left(\eta_c(\xbf_i)-\eta_c(\xbf_{j_i})+\eta_c(\xbf_{j_i})\right)l(\xbf_i,c;\wbf) \nonumber\\
    &\hspace{-7em}\leq \mbbE_{y_{j_i}}[l(\xbf_i,y_{j_i};\wbf)] + \sum_{c=1}^C |\eta_c(\xbf_i)-\eta_c(\xbf_{j_i})|l(\xbf_i,c;\wbf) \label{eq:proof - triangle}\\
    &\leq \mbbE_{y_{j_i}}[l(\xbf_i,y_{j_i};\wbf)] + \delta\lambda_\eta LC,\label{eq:proof - Lipschitz eta}
\end{align}
where \eqref{eq:proof - triangle} is due to the triangle inequality, and \eqref{eq:proof - Lipschitz eta} is due to the assumptions that $\eta_c(\cdot)$ is $\lambda_\eta$-Lipschitz, $\|\xbf_i-\xbf_{j_i}\|\leq \delta$, and $l(\xbf,y;\wbf)$ is upper-bounded by $L$. 
Moreover,
\begin{align}
    \mbbE_{y_{j_i}}[l(\xbf_i,y_{j_i};\wbf)] &= \sum_{c=1}^C \eta_c(\xbf_{j_i}) \left(l(\xbf_i,c;\wbf) - l(\xbf_{j_i},c;\wbf)\right) \nonumber\\
    & \hspace{6em}+ \mbbE_{y_{j_i}}[l(\xbf_{j_i},y_{j_i};\wbf)] \nonumber\\
    &\leq \delta \lambda_l + \mbbE_{y_{j_i}}[l(\xbf_{j_i},y_{j_i};\wbf)],\label{eq:proof - Lipschita l}
\end{align}
where \eqref{eq:proof - Lipschita l} is due to the assumptions that $l(\cdot,c;\wbf)$ is $\lambda_l$-Lipschitz and $\|\xbf_i-\xbf_{j_i}\|\leq \delta$. Plugging \eqref{eq:proof - Lipschita l} into \eqref{eq:proof - Lipschitz eta} yields
\begin{align}
    \mbbE_{y_i}[l(\xbf_i,y_i;\wbf)] &\leq \delta(\lambda_l + \lambda_\eta LC) +  \mbbE_{y_{j_i}}[l(\xbf_{j_i},y_{j_i};\wbf)]. \label{eq:proof - E[li]}
\end{align}
We will utilize the following form of Hoeffding's inequality~\cite{Hoeffding63JASA}: If $X_1,\ldots,X_n$ are independent random variables with $X_i\in [a_i,b_i]$ almost surely, then for any $t>0$,
\begin{align}\label{eq:Hoeffding}
    \Pr\{|\sum_{i=1}^n X_i - \sum_{i=1}^n \mbbE[X_i]|\geq t\}\leq 2\exp\left(-{2t^2\over \sum_{i=1}^n (b_i-a_i)^2} \right).
\end{align}
Since given $P$, $\{y_i\}_{i=1}^n$ are (conditionally) independent, the losses $\{l(\xbf_i,y_i;\wbf)\}_{i=1}^n$ are independent random variables within $[0,L]$. Applying Hoeffding's inequality, we have
\begin{align}
    \Pr\{\left|{1\over n}\sum_{i=1}^n l(\xbf_i,y_i;\wbf)-{1\over n}\sum_{i=1}^n \mbbE_{y_i}[l(\xbf_i,y_i;\wbf)] \right|\geq {t\over n}\} \nonumber\\
    \leq 2\exp\left(-{2 t^2\over n L^2} \right).
\end{align}
Setting $2\exp\left(-{2 t^2/ (n L^2)} \right)=\gamma$, we have $t/n=\sqrt{L^2\log(2/\gamma)/(2n)}$, i.e., with a probability of at least $1-\gamma$,
\begin{align}
    &{1\over n}\sum_{i=1}^n l(\xbf_i,y_i;\wbf) \leq {1\over n}\sum_{i=1}^n \mbbE_{y_i}[l(\xbf_i,y_i;\wbf)] + \sqrt{{L^2\log(2/\gamma)\over 2n}} \nonumber\\
    & \leq \delta(\lambda_l \hspace{-.25em}+\hspace{-.25em} \lambda_\eta LC) \hspace{-.25em}+\hspace{-.25em} \sqrt{{L^2\log(2/\gamma)\over 2n}} +\hspace{-.5em} \sum_{j: \xbf_j\in S}\hspace{-.25em}{u_j\over n}\mbbE_{y_j}[l(\xbf_j,y_j; \wbf)], \label{eq:proof - avg training loss}
\end{align}
where \eqref{eq:proof - avg training loss} is by plugging in \eqref{eq:proof - E[li]} and noting that 
\begin{align}
 {1\over n}\sum_{i=1}^n \mbbE_{y_{j_i}}[l(\xbf_{j_i},y_{j_i};\wbf)] = \sum_{j: \xbf_j\in S}{u_j\over n}\mbbE_{y_j}[l(\xbf_j,y_j; \wbf)].    \nonumber
\end{align}
Since given $S$, 
\(\{y_j:\: \xbf_j\in S\}\) are (conditionally) independent, the weighted losses 
\(\{{(u_j/ n)}l(\xbf_j,y_j;\wbf):\: \xbf_j\in S\}\) 
are independent random variables with \({(u_j/ n)}l(\xbf_j,y_j;\wbf)\in [0,u_j L/n].\)
Applying Hoeffding's inequality again,  we have
\begin{align}
&    \Pr\left\{\hspace{-.05em} \left|\hspace{-.0em}\sum_{j: \xbf_j\in S} \hspace{-.5em}{u_j\over n}l(\xbf_j,y_j;\wbf) \hspace{-.05em}-\hspace{-.5em} \sum_{j: \xbf_j\in S}\hspace{-.5em}{u_j\over n}\mbbE_{y_j}[l(\xbf_j,y_j;\wbf)] \right|\hspace{-.25em}\geq\hspace{-.25em} t\right\} \nonumber\\
 &\hspace{9em}   \leq 2\exp\left(-{2 t^2 n^2\over L^2 \sum_{j: \xbf_j\in S}u_j^2} \right). \label{eq:proof - apply Hoeffding}
\end{align}
Since setting $t= \sqrt{L^2 \log(2/\gamma)(\sum_{j:\xbf_j\in S}u_j^2)/(2 n^2)}$ makes the right-hand side of \eqref{eq:proof - apply Hoeffding} equal to $\gamma$, we have that with a probability of at least $1-\gamma$, 
\begin{align}
    \sum_{j: \xbf_j\in S} {u_j\over n}\mbbE_{y_j}[l(\xbf_j,y_j;\wbf)] &\leq \sum_{j: \xbf_j\in S} {u_j\over n}l(\xbf_j,y_j;\wbf) \nonumber\\
    &+ \sqrt{{L^2 \log(2/\gamma)\over 2 n^2} \sum_{j:\xbf_j\in S}u_j^2}. \label{eq:proof - mean coreset loss}
\end{align}
Combining \eqref{eq:proof - mean coreset loss} with \eqref{eq:proof - avg training loss} yields the desired result. 
Note that we have assumed the labels for $P$ and the labels for $S$ to be conditionally independent (for given $P$ and $S$), i.e., ${1\over n}\sum_{i=1}^n l(\xbf_i,y_i; \wbf)$ is the (random) loss over $P$ based on a set of possible labels for $P$, and $\sum_{j: \xbf_j\in S} {u_j\over n}l(\xbf_j,y_j;\wbf)$ is the (random) loss over $S$ based on a set of independently generated labels for $S$. 
\end{proof}

\begin{proof}[Proof of Theorem~\ref{thm:delta-cover guarantee}]
The bound $\|\tilde{\ubf}\|\leq {\kappa\over \sqrt{|S|}}$ is enforced by design. Algorithm~\ref{Alg:predictive coreset} ensures that the number of candidate samples 
covered by each $\xbf_j\in S$ satisfies 
$u_j \leq \kappa$. Thus, 
\begin{align}
    \|\tilde{\ubf}\| = \sqrt{\sum_{j: \xbf_j\in S}\left({u_j\over nT+s_0}\right)^2} \leq {\kappa \sqrt{|S|}\over nT+s_0} \leq {\kappa \over \sqrt{|S|}},
\end{align}
where the last inequality is because $|S|\leq nT+s_0$. 

We now prove the second statement. 
As Algorithm~\ref{Alg:predictive coreset} selects the samples in $S^{(1)}$ based on the existing samples in $S^{(0)}$ and the predicted samples in $\hat{P}$, all the subsequent probabilities denote conditional probabilities under the given $S^{(0)}$ and $\hat{P}$, where the conditions are omitted for simplicity. Let $[n]:=\{1,\ldots,n\}$. What we need to show is that
\begin{align}
    &\Pr\{\forall \xbf_i\in P,  \|\xbf_i-\xbf_{j_i}\| \leq \delta\} \nonumber\\
    &= \prod_{i\in [n]: \xbf_{j_i}\in S^{(0)}} \Pr\{\|\xbf_i-\xbf_{j_i}\| \leq \delta\} \nonumber\\
& \hspace{1.5em}    \cdot \prod_{j: \xbf_j\in S^{(1)}}\Pr\bigg\{\bigcap_{i\in [n]: j_i=j}\big(\|\xbf_i-\xbf_j\| \leq \delta \big)\bigg\} \label{eq:proof of delta-cover - expansion}\\
    &    \geq 1-\epsilon, \label{eq:proof of delta-cover 1}
\end{align}
where $\xbf_{j_i}$ is the point in $S$ that Algorithm~\ref{Alg:predictive coreset} uses to cover $\hat{\xbf}_i$. The decomposition in \eqref{eq:proof of delta-cover - expansion} is due to the assumption~\eqref{assumption:prediction error}, which implies that $\xbf_1,\ldots,\xbf_n$ are (conditionally) independent given $\hat{P}$. 
To prove \eqref{eq:proof of delta-cover 1}, it suffices to show that
\begin{align}
&    \Pr\{\|\xbf_i-\xbf_{j_i}\| \leq \delta \}\geq (1-\epsilon)^{1/n},~\forall i\in I_0, \label{eq:proof of delta-cover main - S0} \\
&    \Pr\bigg\{\bigcap_{i\in I_j}\big(\|\xbf_i-\xbf_j\| \leq \delta \big)\bigg\} \geq (1-\epsilon)^{|I_j|/n},~\forall j: \xbf_j\in S^{(1)}, \label{eq:proof of delta-cover main - S1}
\end{align}
where $I_j:=\{i\in [n]: j_i=j\}$ ($\forall j: \xbf_j\in S^{(1)}$) and $I_0 :=\{i\in [n]: \xbf_{j_i}\in S^{(0)}\}$. This is because \eqref{eq:proof of delta-cover main - S0} implies that the first product in \eqref{eq:proof of delta-cover - expansion} $\geq (1-\epsilon)^{|I_0|/n}$, \eqref{eq:proof of delta-cover main - S1} implies that the second product in \eqref{eq:proof of delta-cover - expansion} $\geq (1-\epsilon)^{(\sum_{j:\xbf_j\in S^{(1)}}|I_j|)/n}$, and $|I_0| + \sum_{j:\xbf_j\in S^{(1)}}|I_j| = n$. 
We separately consider the cases of $\xbf_{j_i}\in S^{(0)}$ and $\xbf_{j_i}\in S^{(1)}$. 
Let $\xibf_i:= \xbf_i-\hat{\xbf}_i$ denote the prediction error for the $i$-th sample in the prediction window. 
If $\xbf_{j_i}\in S^{(0)}$, then by the triangle inequality and the design of Algorithm~\ref{Alg:predictive coreset}, we have $\|\xbf_i-\xbf_{j_i}\|\leq \|\xibf_i\| + \|\hat{\xbf}_i-\xbf_{j_i}\| \leq \|\xibf_i\| + \delta_0$, and hence
\begin{align}
    \Pr\{ \|\xbf_i-\xbf_{j_i}\| \leq \delta \} \geq 
    \Pr\{\|\xibf_i\| \leq \delta - \delta_0\}. \label{eq:proof of delta-cover 2}
\end{align}
Similarly, if $\xbf_{j_i}\in S^{(1)}$, then by the triangle inequality and the design of Algorithm~\ref{Alg:predictive coreset}, we have $\|\xbf_i-\xbf_{j_i} \|\leq \|\xibf_i-\xibf_{j_i}\|+\|\hat{\xbf}_i-\hat{\xbf}_{j_i}\|\leq \|\xibf_i-\xibf_{j_i}\| + \delta_1$, and hence $\forall j$ with $\xbf_j\in S^{(1)}$,
\begin{align}
   &\Pr \hspace{-.25em} \bigg\{ \hspace{-.25em} \bigcap_{i\in I_j} \hspace{-.25em} \big(\|\xbf_i-\xbf_j\| \leq \delta\big) \hspace{-.25em} \bigg\} \geq \Pr \hspace{-.25em} \bigg\{ \hspace{-.25em} \bigcap_{i\in I_j} \hspace{-.25em} \big( \|\xibf_i-\xibf_j\| \leq \delta-\delta_1 \big) \hspace{-.25em} \bigg\}. \label{eq:proof of delta-cover 3}
\end{align}

We will leverage the fact that for any $\xibf\sim \mathcal{N}(\bm{0}, \sigma^2\bm{I})\in \mathbb{R}^d$, 
\begin{align}
    \Pr\{\|\xibf\|\leq r\} = F\left({r^2\over \sigma^2}; d\right), \label{eq:chi-squared}
\end{align}
where $F(\cdot;d)$ is the CDF of $\chi^2(d)$, the chi-squared distribution with $d$ degrees of freedom. This is because 
\begin{align}
  \Pr\{\|\xibf\|\leq r\} &= \Pr\{\|\xibf\|^2 = \sum_{i=1}^d \xi_i^2\leq r^2\} \nonumber\\
&   = \Pr\{\sum_{i=1}^d \left( {\xi_i\over \sigma}\right)^2\leq {r^2\over \sigma^2}\} = F\left({r^2\over \sigma^2}; d\right),
\end{align}
where we have used the fact that $\sum_{i=1}^d \left( {\xi_i\over \sigma}\right)^2$ is distributed by $\chi^2(d)$ since $\{{\xi_i/ \sigma}\}_{i=1}^d$ are i.i.d. standard Gaussian variables. 

Applying \eqref{eq:chi-squared} to \eqref{eq:proof of delta-cover 2} yields that $\forall i\in I_0$ (i.e., $\xbf_{j_i}\in S^{(0)}$),
\begin{align}
    \Pr\{ \|\xbf_i-\xbf_{j_i}\| \leq \delta \}  &\geq F\left({(\delta- \delta_0)^2\over \sigma_n^2}; d\right) \nonumber\\
    &= (1-\epsilon)^{1/n} \label{eq:proof of delta-cover sub1}
\end{align}
by plugging in the definition of $\delta_0$ in \eqref{eq:delta_0}. 
Moreover, 
\begin{align}
    &\Pr \hspace{-.25em} \bigg\{ \hspace{-.25em} \bigcap_{i\in I_j} \hspace{-.25em} \big( \|\xibf_i-\xibf_j\| \leq \delta-\delta_1 \big) \hspace{-.25em} \bigg\} \nonumber\\
    &=\mbbE_{\xibf_j} \left[\Pr \hspace{-.25em} \bigg\{ \hspace{-.25em} \bigcap_{i\in I_j} \hspace{-.25em} \big( \|\xibf_i-\xibf_j\| \leq \delta-\delta_1 \big) \Big|\xibf_j \hspace{-.25em} \bigg\} \right] \nonumber\\
    &\geq \mbbE_{\xibf_j} \left[\bigg(\Pr\Big\{\|\xibf_i-\xibf_j\|\leq \delta-\delta_1 \Big|\xibf_j \Big\} \bigg)^{|I_j|} \right] \label{eq:proof of delta-cover - iid} \\
    &\geq \left(\mbbE_{\xibf_j} \Big[\Pr\big\{\|\xibf_i-\xibf_j\|\leq \delta-\delta_1\Big| \xibf_j \big\} \Big] \right)^{|I_j|} \label{eq:proof of delta-cover - Jensen} \\
    &= \left(\Pr\big\{\|\xibf_i-\xibf_j\|\leq \delta-\delta_1 \big\} \right)^{|I_j|} \nonumber\\
    &=\left(F\Big({(\delta-\delta_1)^2 \over 2\sigma_n^2}; d\Big) \right)^{|I_j|} \label{eq:proof of delta-cover - chi-squared} \\
    &= (1-\epsilon)^{|I_j|/n}, \label{eq:proof of delta-cover - delta1}
\end{align}
where \eqref{eq:proof of delta-cover - iid}, defined for an arbitrary $i\in I_j\setminus \{j\}$, is because $\{\|\xibf_i-\xibf_j\|\}_{i\in I_j\setminus \{j\}}$ are conditionally i.i.d. given $\xibf_j$ (the ``$\geq$'' is due to replacing $|I_j\setminus \{j\}|$ by $|I_j|$), \eqref{eq:proof of delta-cover - Jensen} is by applying Jensen's inequality to the convex function $(\cdot)^{|I_j|}$, \eqref{eq:proof of delta-cover - chi-squared} is by applying \eqref{eq:chi-squared} to $\Pr\big\{\|\xibf_i-\xibf_j\|\leq \delta-\delta_1 \big\}$ (note that $\xibf_i-\xibf_j\sim\mathcal{N}(\bm{0}, 2\sigma_n^2\bm{I})$ for $i\neq j$), and \eqref{eq:proof of delta-cover - delta1} is by plugging in the definition of $\delta_1$ in \eqref{eq:delta_1}. 
As \eqref{eq:proof of delta-cover sub1} proves \eqref{eq:proof of delta-cover main - S0} and
\eqref{eq:proof of delta-cover 3} together with \eqref{eq:proof of delta-cover - delta1} proves \eqref{eq:proof of delta-cover main - S1}, we have proved that
\[\Pr\{\forall \xbf_i\in P,  \|\xbf_i-\xbf_{j_i}\| \leq \delta\} \geq 1-\epsilon,\]
which completes the proof.
\end{proof}

\fi

\end{document}